\pdfoutput=1
\documentclass{article}

\usepackage[main, final]{neurips_2026}

\usepackage[utf8]{inputenc} 
\usepackage[T1]{fontenc}    
\usepackage{hyperref}       
\usepackage{url}            
\usepackage{booktabs}       
\usepackage{amsfonts}       
\usepackage{amsmath}        
\usepackage{amssymb}        
\usepackage{nicefrac}       
\usepackage{microtype}      
\usepackage{xcolor}         
\usepackage{tikz-cd}
\usepackage{todonotes}
\usepackage{subcaption}
\usepackage{enumitem}
\usepackage{amsthm}    

\newtheorem{theorem}{Theorem}

\newtheorem{corollary}[theorem]{Corollary}
\newtheorem*{remark}{Remark}   

\usepackage{algorithm}     
\usepackage{algpseudocode} 
\usepackage{tikz}
\usepackage{graphicx}
\usepackage{xcolor}
\usepackage{amsmath,amssymb}
\usetikzlibrary{arrows.meta, calc, positioning}

\definecolor{plotblue}{HTML}{1d4ed8}
\definecolor{labelgray}{HTML}{3a3a3a}
\definecolor{lightgray}{HTML}{b0b0b0}
\definecolor{annogray}{HTML}{666666}
\definecolor{gtgreen}{HTML}{16a34a}
\definecolor{inputamber}{HTML}{d97706}

\usetikzlibrary{positioning, calc}

\definecolor{cCircle}{HTML}{D65F5F}
\definecolor{cBump}{HTML}{6ACC65}
\definecolor{cMNIST}{HTML}{EE854A}
\definecolor{cMVCNN}{HTML}{956CB4}
\definecolor{cConf}{HTML}{8C613C}
\definecolor{labelgray}{HTML}{555555}

\usepackage{wrapfig}
\usepackage{orcidlink}      

\tikzset{
  ptitle/.style = {font=\sffamily\bfseries\normalsize, anchor=west},
  pdef/.style   = {font=\small, text=labelgray, anchor=west},
  pnote/.style  = {font=\scriptsize, text=labelgray, anchor=north},
  img/.style    = {inner sep=0pt},
}

\usepackage{defs}

\newtheorem*{theoremRef}{Theorem \ref{thm:main} (bis)}

\title{It Just Takes Two:\\Scaling Amortized Inference to Large Sets}
\workshoptitle{It Just Takes Two:\\Scaling Amortized Inference to Large Sets}

\author{Antoine Wehenkel\thanks{Equal Contribution} ~ \orcidlink{0000-0001-5022-3999}\\
  Apple\\
  \texttt{awehenkel@apple.com}
  \And
  Michael Kagan
  \orcidlink{0000-0002-3386-6869}\\
  SLAC National Accelerator Laboratory\\
  \texttt{makagan@slac.stanford.edu}
  \And
  Lukas Heinrich
  \orcidlink{0000-0002-4048-7584}\\
  TU M\" unchen \\
  \texttt{Lukas.Heinrich@cern.ch}
  \And
  Chris Pollard$^*$
  \orcidlink{0000-0002-3690-3960}\\
  University of Warwick \\
  \texttt{christopher.pollard@warwick.ac.uk}
}

\begin{document}

\maketitle

\begin{abstract}
  Neural posterior estimation has emerged as a powerful tool for amortized inference, with growing adoption across scientific and applied domains. In many of these applications, the conditioning variable is a \emph{set of observations} whose elements depend not only on the target but also on unknown factors shared across the set. Optimal inference therefore requires treating the set jointly, which in turn requires training the estimator at the deployment set size---a regime where memory and compute quickly become prohibitive. We introduce a simple, theoretically grounded strategy that decouples representation learning from posterior modeling. Our method trains a mean-pool \deepset on sets of size at most two, producing an encoder that generalizes to arbitrary set sizes. The inference head is then finetuned on pre-aggregated embeddings, making training cost essentially independent of the deployment set size $N$. Across scalar, image, multi-view 3D, molecular, and high-dimensional conditional generation benchmarks with $N$ in the thousands, our approach matches or outperforms standard baselines at a fraction of the compute.
\end{abstract}

\section{Introduction}
Many inference problems involve analyzing a large collection of coupled measurements---for instance, extracting a signal strength from an ensemble of particle-collider events or reconstructing a 3D scene from several 2D images. Because each observation contains only partial information, optimal inference requires both processing the set jointly and quantifying the residual uncertainty. As a solution, neural posterior estimation~\citep[NPE,][]{lueckmann2017flexible,wildberger2023flow,geffner2023compositional} has emerged as an effective tool for amortized probabilistic inference. In this specific context, however, effective inference hinges on an architecture capable of processing observation sets whose cardinality $N$ may need to be large.

Two strategies for such tasks dominate the literature, each trading off scalability against correctness. The first trains an estimator on individual measurements and recombines per-measurement scores or likelihoods at test time~\citep{geffner2023compositional,papamakarios2019sequential}. Trivially scalable in $N$, this approach nonetheless leads to incorrect or expensive inference whenever latent factors shared across the set---unrelated to the parameter of interest, such as a common detector calibration or scene illumination---couple the measurements. The second strategy instead processes the set jointly by learning a permutation-invariant aggregation of per-measurement features, most commonly via a mean-pool Deep Set encoder~\citep{zaheer2017deep} trained end-to-end at the deployment cardinality~\citep{radev2020bayesflow,wiqvist2019partially,sainsbury2024likelihood,rodrigues2021hnpe}. While principled, end-to-end training has memory and compute footprints that scale linearly with the targeted cardinality $N$ per gradient step, preventing its application to arbitrary $N$ in practice. Practitioners are thus left to choose between a scalable but sub-optimal estimator and a principled but hard-to-scale one.

We show this trade-off stems from conflating two subproblems: learning a permutation-invariant data representation, and modeling the posterior given that representation. Only the former has a training cost that scales with N, and, as we prove, it can be learned from sets of size at most two without loss of performance---under mean-pool aggregation, a reasonable restriction given its prevalence in practice. We turn this insight into a simple procedure named PAIRS (Pretraining Aggregators for Inference at aRbitrary Set-sizes). First, PAIRS jointly pretrains a mean-pool \deepset encoder and an inference head on sets of size $N \leq 2$. Second, it freezes the encoder and finetunes the inference head on embeddings pre-aggregated from sets of arbitrary size. Our contributions are (i) a theoretical result showing that pair-training recovers the same encoder (up to affine reparameterization) as training at any larger cardinality; (ii) the PAIRS procedure, whose training cost is decoupled from the deployment set size and which applies to any mean-pool \deepset pipeline; and (iii) empirical validation across scalar, image, 3D, molecular, and image-generation tasks with N in the thousands, matching or outperforming baselines at a fraction of the compute.
\section{Background and Methods}\label{sec:method}
We consider the task of predicting a target quantity, denoted $\theta \in \Theta$, given a set of $N$ observations $X := \{x_i\}^N_{i=1} \in \mathcal{X}^N$. Our goal is to build an amortized probabilistic predictor that can easily be applied on-the-fly to a new observation set $X$ of arbitrary size $N \in \mathcal{N} \subseteq \mathbb{Z}^+$. Taking a Bayesian modeling perspective, we formalize this task as building a neural surrogate $q(\theta \mid X)$ of the posterior distribution $p(\theta \mid X) \propto p(\theta, X)$ implied by the assumed hierarchical model 
\begin{equation}
    p(\theta, X) := p(\theta)\int p(\psi) \prod^N_{i=1} p(x_i \mid \theta, \psi)~\text{d}\psi, \label{eq:hierarchical}
\end{equation}
where $\psi \in \Psi$ parameterizes nuisance effects that are not of direct interest but impact observations collectively and must be accounted for.

While this hierarchical model ignores potential dependence between $N$ and ($\theta$, $\psi$), it provides a practical way to formalize inference given an observation set $X$ that is compatible with most real-world scenarios. \autoref{eq:hierarchical} shows that marginalizing $\psi$ couples the observations, so the posterior does not simply factorize across them. As also noted by \citet{Heinrich:2023bmt}, a neural surrogate exposed only to single $(x_i, \theta)$ pairs during training therefore has no way to discover this coupling.

\subsection{Neural Posterior Estimation and Deep Sets}
Our goal is to train an NPE~\footnote{We use ``NPE'' for any deep generative posterior model, not only normalizing-flow-based ones.} that can condition on sets of
arbitrarily large size $N$. Training such an estimator end-to-end quickly
becomes infeasible, because the memory and compute cost of every training step
grows with $N$. Our solution, PAIRS (\autoref{sec:pairs}), avoids this by
pretraining the set encoder on sets of at most two elements. Before presenting
it, we review the two components it builds on, NPE and Deep Sets, and explain
why standard training does not scale.

In the context of simulation-based inference~\citep[SBI,][]{cranmer2020frontier}, NPE\citep[see, e.g.,][]{lueckmann2017flexible, geffner2023compositional, wildberger2023flow} has become a popular algorithm to learn a neural surrogate of $p(\theta \mid X)$ for amortized Bayesian inference.  In short, NPE parameterizes the posterior with a conditional generative model~(CGM), which can be decomposed into (1) a neural statistic estimator~(NSE), denoted by $T_{\omega}: \mathcal{X}^N \rightarrow \mathbb{R}^l$, compressing observations $X$ into $l\text{-dimensional}$ representations, and (2) a deep probabilistic model~\citep[DPM,][]{wehenkel2022inductive} that parameterizes the posterior as $q_{\phi}(\theta \mid T_{\omega}(X))$. 

 For the DPM class, we use normalizing flows~\citep[NF,][]{papamakarios2021normalizing} for all experiments except one where $\theta$ is an image, for which flow matching~\citep[FM,][]{lipman2022flow} is more appropriate. While the two classes differ in their practical training surrogates, both can be understood as  solving the following optimization problem 
 \begin{align}
    \phi^\star, \omega^\star\in& \arg\max_{\phi, \omega} \mathbb{E}_{\theta, X \sim p(\theta, X)} \left[\log q_{\phi}(\theta \mid T_\omega(X))\right]. \label{eq:NPE}
\end{align}  
\emph{Remark:} NF optimizes (2) directly via exact likelihoods; FM optimizes a conditional vector-field regression loss that upper-bounds the same Kullback-Leibler divergence [Lipman et al., 2022].

In practice, embedding a proper inductive bias into the NSE $T_\omega$ is key to learning a good surrogate of $p(\theta \mid X)$. The hierarchical latent model from \autoref{eq:hierarchical} implies exchangeability of the elements $x_i$ of $X$, making \deepset architectures a natural choice. In this work, we consider mean-pool Deep Sets, which parameterize the NSE as 
\begin{align}
    T_{\omega}(X) = \frac{1}{N} \sum^N_{i=1} t_{\omega}(x_i), \label{eq:deep_set}
\end{align}
where $t_{\omega}: \mathcal{X} \rightarrow \mathbb{R}^l$ is a neural network mapping an individual observation $x_i$ to an $l$-dimensional vector. 

As noted by \citet{rodrigues2021hnpe, Heinrich:2023bmt, sainsbury2024likelihood}, a surrogate $q(\theta\mid X)$ trained at a fixed set size does not generalize to other sizes. The standard remedy is to randomize $N \sim p(N)$ and give $N$ as an extra input to the CGM during training, which becomes $q_{\phi}(\theta \mid T_{\omega}, N)$. This training scheme requires $p(N)$ to have support up to the largest size $N_{\max}$ expected at test time. The memory and compute footprint of each gradient step on a batch of $B$ sets then scales as $\mathcal{O}(B\, N_{\max})$ (peak memory) and $\mathcal{O}(B\, \bar{N})$ (compute), with $\bar{N} \triangleq \mathbb{E}_{p(N)}[N]$. Since training requires many such steps, these scaling laws rule out arbitrarily large $N$ under a fixed hardware budget.

\subsection{Representation Learning and Exponential Families}
\label{sec:pairs}
As a solution to the scaling issue exposed in the previous section, we introduce \textit{Pretraining Aggregators for Inference at aRbitrary Set-sizes}~(PAIRS), a three-stage procedure described by Algorithm~\ref{alg:pairs}. First, PAIRS pretrains the \deepset NSE and the  DPM jointly, by optimizing \autoref{eq:NPE} on observation sets composed of $N\leq 2$ elements. Second, PAIRS embeds each individual measurement via the learned \deepset encoder $t_{\omega}$. Third, PAIRS finetunes the DPM on the per-set aggregated embeddings following an appropriate set size distribution $p(N)$. Crucially, since the encoder is frozen and embeddings are pre-aggregated, the memory and compute footprint of each finetuning step is independent of $N_{\text{max}}$.

\begin{algorithm}[ht]
\caption{PAIRS: Pretraining Aggregators for Inference at aRbitrary Set-sizes}
\label{alg:pairs}
\begin{algorithmic}[1]
\Require Pretraining dataset $\mathcal{D}_{\text{pre}} = \{(\theta^{(b)}, \psi^{(b)}, X^{(b)})\}_b$ with $|X^{(b)}| \in \{1,2\}$;
         finetuning dataset $\mathcal{D}_{\text{ft}} = \{(\theta^{(j)}, \psi^{(j)}, X^{(j)})\}_{j=1}^{M}$ with $|X^{(j)}| \sim p(N)$;
         per-observation embedder $t_\omega$;
         conditional generative model $q_\phi(\theta \mid T_\omega, N)$ with loss $\ell_{\phi}: \Theta \times \mathbb{R}^l \times \mathcal{N} \rightarrow \mathbb{R}$.
\Statex
\Statex \textbf{Stage 1 -- Pretrain the embedder at small cardinality.}
\Repeat
  \State Sample a minibatch $\{(\theta^{(b)}, X^{(b)})\}_{b=1}^{B} \sim \mathcal{D}_{\text{pre}}$
  \State Compute aggregates $T^{(b)}_{\omega} = \tfrac{1}{|X^{(b)}|}\sum_{x \in X^{(b)}} t_\omega(x)$
  \State Update $(\omega, \phi)$ by ascending
         $\nabla_{\omega,\phi}\, \tfrac{1}{B}\!\sum_b \ell_{\phi}(\theta^{(b)}, T_{\omega}^{(b)}, |X^{(b)}|)$
\Until{converged}
\Statex
\Statex \textbf{Stage 2 -- Cache aggregated embeddings on the finetuning set.}
\For{$j = 1, \ldots, M$}
  \State $\bar{T}^{(j)} \gets \tfrac{1}{|X^{(j)}|} \sum_{x \in X^{(j)}} t_\omega(x)$
         \Comment{single forward pass per observation, computed once}
\EndFor
\State Store $\mathcal{D}_{\text{agg}} = \{(\theta^{(j)}, \bar{T}^{(j)}, |X^{(j)}|)\}_{j=1}^{M}$.
\Statex
\Statex \textbf{Stage 3 -- Finetune the inference head on cached embeddings.}
\Repeat
  \State Sample a minibatch $\{(\theta^{(b)}, \bar{T}^{(b)}, N^{(b)})\}_{b=1}^{B} \sim \mathcal{D}_{\text{agg}}$
  \State Update $\phi$ by ascending
         $\nabla_{\phi}\, \tfrac{1}{B}\!\sum_b \ell_{\phi}(\theta^{(b)}, \bar{T}^{(b)}, N^{(b)})$
\Until{converged}
\State \Return $(t_\omega,\, q_\phi)$
\end{algorithmic}
\end{algorithm}

While the compute and memory complexity of embedding all measurements is linear in the dataset size, this is a one-off, parallelizable operation. Unlike standard approaches, which jointly amortize $N$ over the NSE and DPM and incur gradient step costs that grow with $\nmax$, PAIRS has memory and compute costs at training time that are independent of  $\nmax$.

Intuitively, under the conditional-iid model of \autoref{eq:hierarchical}, the joint density at any cardinality is fully determined by the per-event likelihood $p(x \mid \theta, \psi)$ and the prior on $\psi$. A single labeled observation $(x, \theta)$ reveals only the marginal $p(x \mid \theta)$, while pairs $(\{x_1, x_2\}, \theta)$ are the smallest cardinality that also exposes the coupling induced by the shared nuisance. Moreover, larger cardinalities add no new information about the underlying generative process. \autoref{thm:main} makes this precise: under the assumption that the data admits a finite-dimensional mean-pool sufficient statistic, PAIRS' optimum coincides with that of training strategies using $N > 2$.
\begin{theorem}[Pair training recovers a mean-pool sufficient statistic]
\label{thm:main}
Suppose there exists a continuous $t^\star : \mathcal{X} \to \mathbb{R}^k$ with
$k \leq l$ such that, for every $N \geq 1$,
\begin{equation}
\big(M_{t^\star}(X_N),\, N\big) \;:=\; \Big(\tfrac{1}{N}\sum_{i=1}^N t^\star(x_i),\, N\Big)
\quad \text{is minimal sufficient for } \theta \text{ given } X_N := \{x_i\}_{i=1}^N.
\tag{M}
\end{equation}
Under mild regularity on $t_\omega$ (Appendix~\ref{app:proof}, assumptions A1--A4 and (S)),
and under the idealization that the optimization problem in Equation~\ref{eq:NPE}
is solved globally at $N \in \{1, 2\}$, the aggregate $(T_\omega(X_N), N)$
is sufficient for $\theta$ at every cardinality $N \geq 1$.
\end{theorem}
A formal statement and proof of Theorem~\ref{thm:main} are given in Appendix~\ref{app:proof}. In short, the proof relies on two main elements of PAIRS. First, the surrogate family must contain the true posteriors, which is a reasonable assumption given the expressivity of modern neural network architectures. Second, under optimal training, this assumption directly translates into sufficiency of the representations found for $N=1$ and $N=2$. The definition of sufficient statistics then leads to a Cauchy equation, which implies an affine form, and thus extrapolation to arbitrary $N$ under mean pooling. Two interesting corollaries follow from proving \autoref{thm:main}. First, Corollary~\ref{cor:ef} indicates that using a mean-pool \deepset is equivalent to assuming $p(X_N \mid \theta)$ belongs to the exponential family. Second, Corollary~\ref{cor:invariance} shows that any training cardinality $N \geq 2$ yields the same representation (up to affine equivalence), so enlarging it beyond $\{1, 2\}$ is not necessary.

\subsection{From Theory to Practice}\label{sec:bridging}
Theorem~\ref{thm:main} rests on three idealizations that we now examine in turn.
\paragraph{(i) The model admits a $k$-dimensional mean-pool sufficient statistic $t^\star$.} Under standard regularity, Corollary~\ref{cor:ef} shows this architectural assumption is equivalent to the $\psi$-marginal likelihood, $p(X\mid \theta)$, being exponential family~(EF). Mean-pool Deep~Sets thus implicitly restrict amortized inference to EF marginals---a restriction inherent to the architecture, not introduced by PAIRS. 
While \citet{wagstaff2019limitations,wagstaff2022universal} demonstrated that the embedding dimension $l$ needs to grow with $N$ to achieve universality, Corollary~\ref{cor:ef} rather exposes the implicit assumption hidden behind using mean-pool Deep~Sets to process observation sets of arbitrary sizes. 
Although this implicit bias is what allows PAIRS to work, it is milder than it appears. 
The EF condition applies to the $\psi$-marginal $p(X \mid \theta)$, not to the per-observation conditional $p(x \mid \theta, \psi)$. Observations within a set can therefore remain strongly coupled---the coupling is carried by the base measure $h_N(X)$ and reflects the shared-$\psi$ structure of the hierarchical model---even when each $p(x \mid \theta, \psi)$ is simple. What EF constrains is only \emph{how} $\theta$ enters the marginal: additively through $\sum_i S(x_i)$. Moreover, EF likelihoods asymptotically approximate a broad class of non-EF likelihoods~\citep{10.1214/aos/1176348252}, i.e.~there exists an embedding dimension $l$ to approximate $p(\theta \mid X)$ well even when the model is not globally EF.
\paragraph{(ii) The \deepset encoder is expressive enough.} This bundles two requirements from the proof of Theorem~\ref{thm:main}: the embedder family $t_\omega$ is a universal approximator of continuous maps $\mathcal X\to\mathbb R^l$ (assumption~A1), and the embedding width satisfies $l\ge k$ (otherwise no linear map $\mathbb R^l\to\mathbb R^k$ can recover $t^\star$ from $t_\omega$ at $N=1$). The first is a mild assumption for modern neural architectures. The second is a hyperparameter under the practitioner's control: as $k$ is unknown a priori, $l$ should be chosen to maximize validation performance of PAIRS. If (i) holds exactly, and given enough data, the hyperparameter search should find $l\ge k$ and close the representation gap. Interestingly, even if (i) does not apply---e.g.\ for order-statistic families with no finite-dimensional mean-pool sufficient statistic---we expect increasing $l$ to improve the approximation between the closest EF marginal and the true non-EF marginal $p(X \mid \theta)$. For instance, a sufficiently wide $t_\omega$ can mimic a discretized empirical CDF through threshold features. Crucially, the cost of growing $l$ is paid once at pretraining and, while it may reduce the compression efficiency, it does not depend on the deployment cardinality $N$, so $l$ remains a cheap but powerful hyperparameter even when (i) is only approximate.
\paragraph{(iii) The NPE problem is solved globally at $N\in\{1,2\}$.} In practice NPE is trained to a local minimum on a finite dataset, not to the global optimum. Two mechanisms make Theorem~\ref{thm:main} robust to this gap. \emph{Cauchy stability.} The proof reduces sufficiency at $N=2$ to the Cauchy functional equation $g(y_1)+g(y_2)=g(y_1+y_2)$, whose continuous solutions are affine. This equation is Hyers--Ulam-stable (i.e., approximate solutions are close to exact ones): if it holds only up to a residual $\varepsilon$, then $g$ is within $O(\varepsilon)$ of an affine map. 
Small training residuals therefore do not destabilize the affine-recovery argument: $t_\omega$ remains close to an affine image of $t^\star$ , and the finetuned head can absorb the affine part directly.
Even when $t_\omega$ is only approximately an affine image of $t^\star$, the aggregate $T_\omega(X_N) = \tfrac{1}{N}\sum_i t_\omega(x_i)$ is an unbiased estimator of the population feature $\bar{t}_\omega(\theta,\psi) := \mathbb E_{p(x\mid\theta,\psi)}[t_\omega(x)]$ with variance $O(1/N)$. Information about $\theta$ thus splits into a \emph{representation} factor--fixed at pretraining and controlled by how identifiable the true sufficient statistic $t^\star$ is--and a \emph{concentration} factor that shrinks with $N$ for free at deployment. Pair training fixes the former; large-$N$ inference improves the latter without any further training of $t_\omega$. More broadly, even without an exact EF-sufficient statistic, PAIRS learns an informative, additively composable representation: the entire set is compressed into a single $l$-dimensional vector with a straightforward mean-update rule, yielding tight posteriors at favorable compute.

\emph{Empirical check.} Figure~\ref{fig:bivgauss} confirms both mechanisms on a bivariate Gaussian model: $x\mid\theta,\psi\sim\mathcal N(\theta,\psi)$ with a Wishart prior on the covariance $\psi$ and a Normal--Wishart prior on $\theta$. This setting satisfies (i) exactly, with $t^\star(x)=(x_1,x_2,x_1^2,x_2^2,x_1x_2)$ minimal sufficient at every $N$, and the analytical posterior is available in closed form. PAIRS, trained exclusively at $N\in\{1,2\}$ and frozen thereafter, tracks the analytical posterior up to $N=10^5$; the NLL decreases monotonically with $N$ and converges to the analytical lower bound of the true posterior, without any further training of $t_\omega$.
\begin{figure}[h]
    \centering
    \begin{subfigure}[t]{0.23\linewidth}
        \centering
        \includegraphics[width=\linewidth]{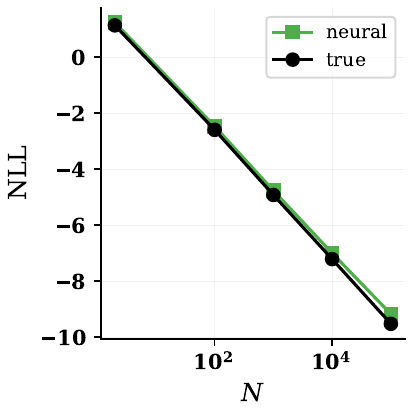}
        \caption{}
        \label{fig:bivgaussscaling}
    \end{subfigure}
    \hspace{.05cm}
    \begin{subfigure}[t]{0.73\linewidth}
        \centering
        \includegraphics[width=\linewidth]{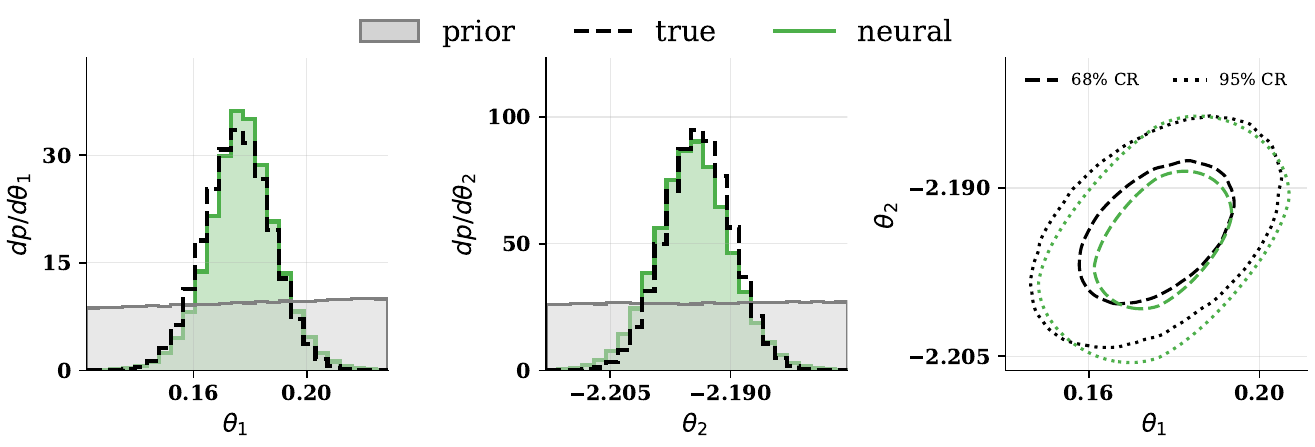}
        \caption{}
        \label{fig:bivgauss10000}
    \end{subfigure}
    \caption{
    PAIRS recovers the analytical posterior on a bivariate Gaussian model, a setting where assumption~(M) holds exactly. (a) Mean test NLL vs.\ set size $N$: PAIRS tracks the analytical posterior up to $N = 10^5$. (b) Neural vs.\ analytical posterior (1D marginals and 2D credible region contours) for a set with $N = 10^4$.
    }
    \label{fig:bivgauss}
\end{figure}
\section{Related Work}\label{sec:related}
Within SBI, two families of scalable approaches dominate. The first composes per-observation objects at test time---marginal likelihoods $p(x_i \mid \theta)$ \citep{papamakarios2019sequential} or posterior scores $\nabla_{\theta}\log p(\theta \mid x_i)$ \citep{geffner2023compositional}---which, as also noted by \citet{Heinrich:2023bmt} and shown in \autoref{eq:hierarchical}, is sub-optimal whenever shared nuisances $\psi$ couple observations, unless $\psi$ is absorbed into $\theta$ and marginalized post-hoc at significant cost. The second learns set representations directly \citep{radev2020bayesflow, wiqvist2019partially, sainsbury2024likelihood, rodrigues2021hnpe}, typically end-to-end at or near the deployment cardinality; \citet{sainsbury2024likelihood} explicitly advocate a curriculum over $N$. These approaches are principled but scale linearly in $N_{\max}$ per gradient step. Orthogonal to the set-size axis, \citet{chen2020neural} learn per-observation near-sufficient statistics via an infomax objective and, like us, invoke Pitman--Koopman--Darmois (see Corollary~\ref{cor:ef}) to motivate a finite-dimensional statistic; their construction, however, targets a single observation and does not address scaling in set cardinality. Our work is complementary as it shows that the representation-learning step itself can be decoupled from the target set size.  

Beyond SBI, Deep Set architectures \citep{zaheer2017deep} have motivated extensive work on scaling to massive sets. A prominent line introduces cross-attention with a small set of learned latent queries \citep{lee2019set, jaegle2021perceiver} to reduce the $\mathcal{O}(N^2)$ cost of full self-attention to $\mathcal{O}(NM)$. These efforts are orthogonal to ours as they target per-step cost at fixed $N$, whereas PAIRS targets training cost as a function of the deployment $N$. In principle, PAIRS can be layered on such architectures, and for the restricted case of set-independent queries---which reduces to weighted mean-pool---we conjecture our theoretical guarantee to carry over directly. For general attention pooling, however, free extrapolation is unlikely~\citep{press2021train, anil2022exploring}: Theorem~\ref{thm:main} crucially exploits the additive, non-interacting structure of mean-pool aggregation, which is lost when aggregation weights are themselves learned functions of the full set. From the theoretical angle, \citet{wagstaff2019limitations, wagstaff2022universal} show that the embedding dimension must grow with $N$ to guarantee universal approximation, reinforcing the practice of training at test-time cardinalities. Our Corollary~\ref{cor:ef}~(\autoref{app:proof}) sharpens this picture: whenever the marginal likelihood is exponential-family, or well approximated by one, a fixed embedding dimension suffices at every $N$, so the worst-case dependence on cardinality can be avoided under modeling assumptions that are mild in many SBI applications.
\section{Experiments}
\label{sec:experiments}
We now empirically assess PAIRS in diverse realistic scenarios where the assumptions underlying \autoref{thm:main} are not necessarily satisfied. We first walk the reader through a toy problem inspired by particle physics, in which observations must be processed jointly to achieve good predictive performance. We then validate the proposed strategy on four low-dimensional parameter-estimation tasks spanning diverse observation modalities, comparing it against baselines a practitioner might consider. Finally, we showcase PAIRS on novel-view synthesis from a set of images, a high-dimensional conditional generative modeling task where end-to-end training at the deployment cardinality is computationally out of reach.

\subsection{Illustrative example: the bump hunt} \label{sec:bump}
\begin{wrapfigure}[16]{l}{0.45\textwidth}  
    \vspace{-1.em}
    \centering
    \includegraphics[width=1.\linewidth]{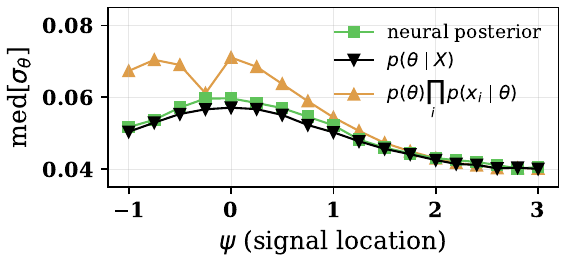}
    \vspace{-1.5em}
    \caption{\small \textbf{PAIRS matches MCMC on the bump hunt, while per-event marginalization fails.} Median posterior standard deviation for $\theta$ as a function of the nuisance signal location $\psi$, at $N = 100$. PAIRS tracks the MCMC posterior width computed on the full set, whereas the na\" ive $\prod_i p(x_i \mid \theta)$ surrogate, which ignores the shared nuisance, is substantially underconfident.}
    \label{fig:bumphunt}
\end{wrapfigure}

\begin{figure}[ht]
    \centering
    \includegraphics[width=.95\linewidth]{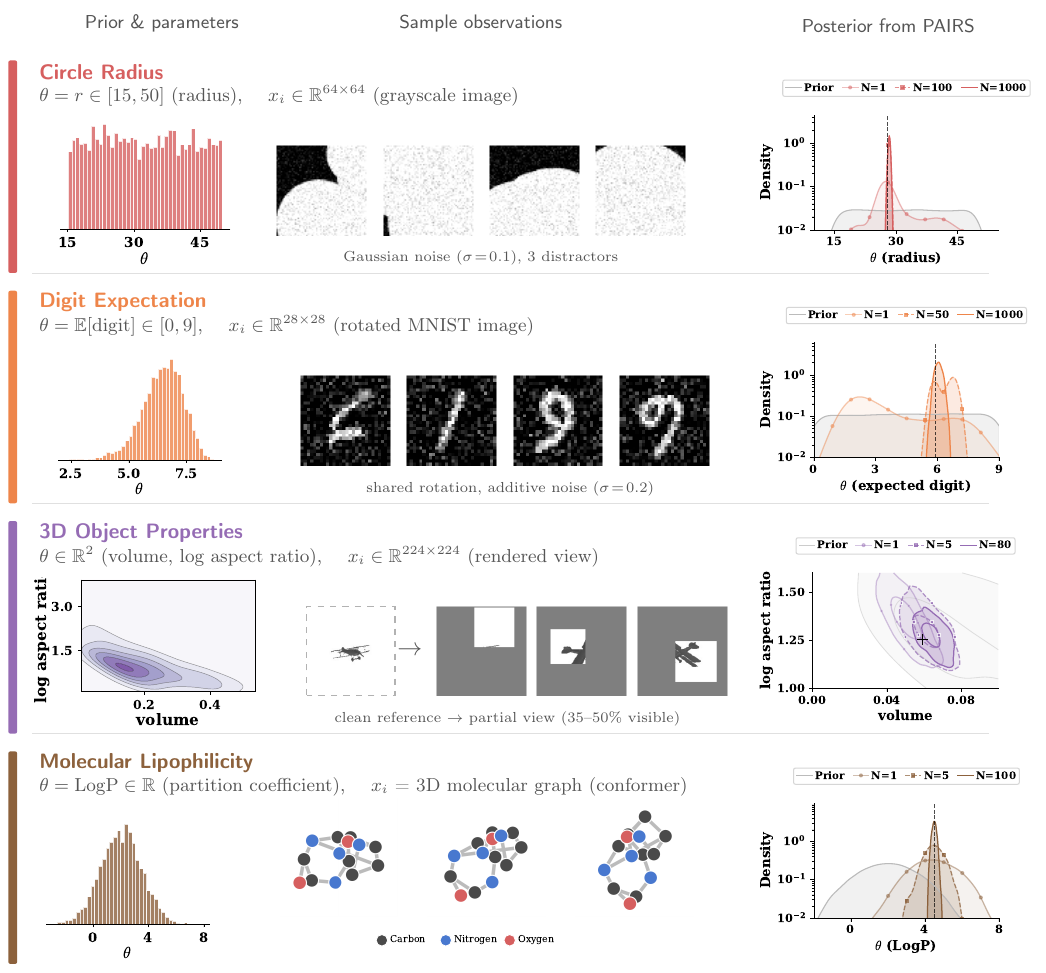}
    \caption{\small \textbf{Evaluation tasks.} Each row shows the prior over the parameter of interest, representative observations at a fixed $\theta$, and the PAIRS posterior as $N$ grows. \textbf{Circle Radius.} Radius $\theta \sim \mathcal{U}(15, 50)$ recovered from $64{\times}64$ images cluttered with distractor circles. The target circle position is the nuisance and multiple observations are needed to properly marginalize it out. \textbf{Digit Expectation.} Expected digit value $\theta = \sum_k k p_k$ from rotated MNIST images sharing a common rotation angle (nuisance). Confusion between $6$ and $9$ creates ambiguity at low $N$, when the nuisance cannot be inferred. \textbf{3D Object Properties.} Volume and log aspect ratio of ModelNet40 objects~\citep{wu20153d} from partial $224{\times}224$ rendered views. \textbf{Molecular Lipophilicity.} Octanol--water partition coefficient (LogP) of drug-like molecules from 3D conformers~\citep{axelrod2022geom}. We provide the full specification of each task in Appendix~\ref{app:exp_details}.}
\label{fig:tasks_description}
\vspace{-1em}
\end{figure}

Resonance searches (``bump hunts'') in particle physics have resulted in the discovery of many fundamental particles, most recently the Higgs boson~\citep{ATLAS:2012yve,CMS:2012qbp}. Following \citet{Heinrich:2023bmt}, we consider a stylized version of this setting in which we infer a signal fraction $\theta \in [0,1]$, treating the signal location $\psi \sim \mathcal{N}(1, 4)$ as a nuisance that globally couples observations. Each observation is drawn from a two-component mixture: a signal component $\mathcal{N}(\psi, 0.1)$ with probability $\theta$ and a background component $\mathcal{N}(0, 1)$ with probability $1 - \theta$. Inferring $\theta$ from a set of observations therefore requires implicitly identifying $\psi$ from their joint distribution. While simple, this model captures the essential structure of the analyses that provided evidence for the Higgs boson. To assess the correctness of the PAIRS-based NPE model, we compare its posterior over $\theta$ to a compute-intensive Markov Chain Monte Carlo~(MCMC) estimate on the full set of observations. \autoref{fig:bumphunt} shows that the PAIRS-learned aggregator tracks the true width of $p(\theta \mid X)$ at $N = 100$ across all tested values of $\psi$, demonstrating that the learned set representation is as informative as the raw observations, even when the likelihood is not in the exponential family and non-trivial marginalization over a shared nuisance is required. PAIRS, moreover, clearly outperforms the na\" ive approach in which the likelihood $p(X \mid \theta)$ is approximated as a product of marginal likelihoods, $\prod_i p(x_i \mid \theta)$.

\subsection{Systematic Evaluation}\label{sec:systematic_evaluation}
We evaluate PAIRS on four tasks described in \autoref{fig:tasks_description} and compare it against natural alternatives a practitioner might consider. Two baselines isolate the role of the set cardinality used to pretrain the Deep Set encoder. First, $N=1$ trains the full model on single observations only, which is effectively equivalent to approximating $p(X \mid \theta)$ as a product of per-observation marginal likelihoods. Then, \textbf{$N=1\text{-}10$} extends pretraining to $N \in \{1, \dots, 10\}$ to gauge the marginal value of a larger pretraining budget. PAIRS itself corresponds to \textbf{$N=1\text{-}2$}. To test whether the probabilistic objective is essential, or whether features learned for point prediction suffice, the \textbf{Gaussian} baseline replaces the normalizing flow with a regression head trained under mean-squared error on $N \in \{1, 2\}$, then fits a per-$N$ standard deviation post hoc to recover an estimated posterior distribution. Finally, on synthetic tasks where data at the maximum cardinality can be generated cheaply, we include an \textbf{end-to-end (E2E)} baseline that trains all components jointly at a fixed target $N = 1000$. All baselines share the same per-task architecture (except Gaussian, which swaps the flow for a regression head); model selection uses validation performance, and each baseline is run with three random seeds.
\begin{figure}[ht]
    \centering
    \includegraphics[width=1.\linewidth]{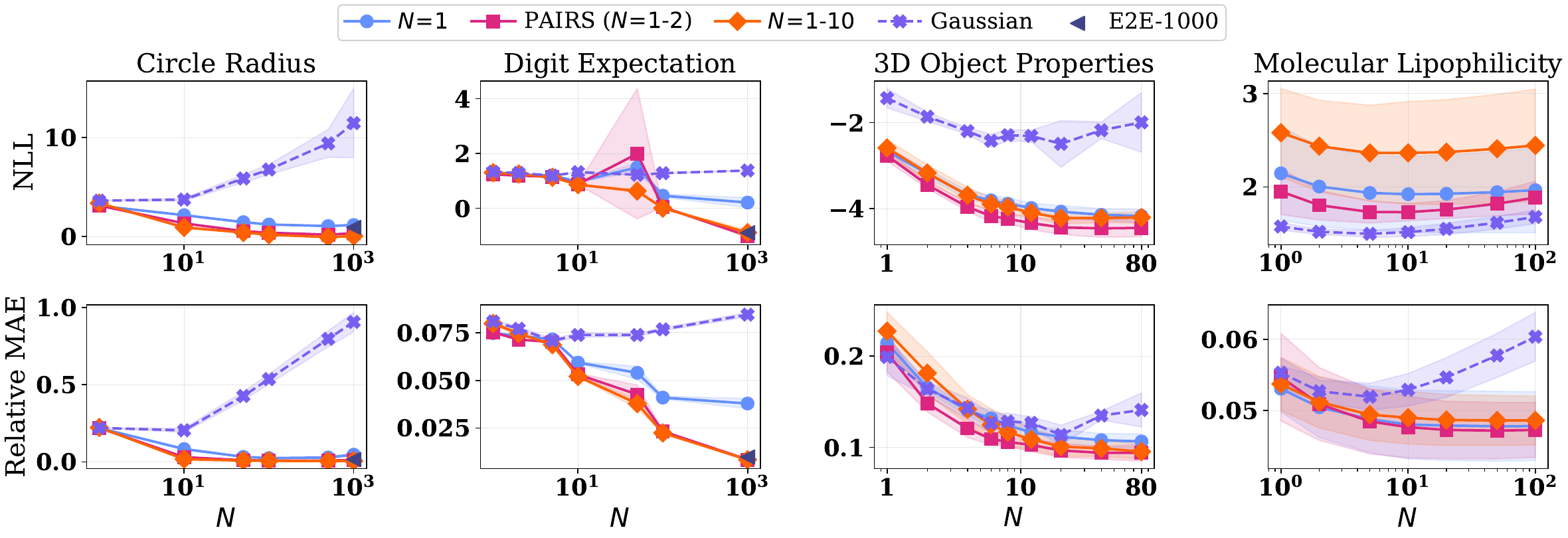}
    \caption{\small \textbf{PAIRS matches or outperforms every baseline on nearly all tasks.} Test NLL (top) and relative MAE (bottom) as a function of $N$ across the four tasks; PAIRS improves monotonically with $N$ despite pretraining only at $N \leq 2$. Shaded regions denote $\pm 1$ std over 3 seeds. }
\label{fig:strategy-comparison}
    \vspace{-1.5em}
\end{figure}

\paragraph{What drives scalable amortized inference.}
\autoref{fig:strategy-comparison} reports the test negative log-likelihood (NLL) and relative mean absolute error (RMAE) as a function of $N$ across all four tasks; calibration metrics are deferred to \autoref{app:ACAUC}. PAIRS delivers consistently strong performance on every benchmark, with both metrics improving steadily as $N$ grows, confirming that the two-stage pipeline amortizes inference to large set cardinalities. Extending pretraining to $N=1\text{--}10$ yields a modest gain only on Circle Radius (NLL $0.05$ vs. $0.35$ at $N = 1{,}000$) and matches or underperforms PAIRS on the other three tasks at $3\text{--}4\times$ the pretraining cost. The empirical results align with the theoretical prediction from  Corollary~\ref{cor:invariance}. The Gaussian baseline, despite seeing the same data as PAIRS, degrades as $N$ grows on Circle Radius and underperforms on Digit Expectation and 3D Object Properties, showing that representations learned for point predictions discard structure that the normalizing flow cannot recover at finetuning time. Finally, the $N=1$ baseline performs surprisingly well on Molecular Lipophilicity and 3D Object Properties, but falls behind on Circle Radius (NLL $1.17$ vs. $0.35$ at $N = 1{,}000$) and Digit Expectation where nuisance coupling dominates. These results empirically confirm the theoretical necessity of $N\ge 2$ training: exposure to $N=2$ is necessary to enforce the affine embedding structure via the Cauchy functional equation (\autoref{thm:main}). More broadly, none of these four tasks admits an EF marginal, yet PAIRS improves monotonically with $N$ throughout---evidence that PAIRS generalizes to large $N$ even when the EF idealization underlying Theorem~\ref{thm:main} does not hold.

\paragraph{Cost-performance tradeoff.}
\begin{wrapfigure}[18]{r}{0.55\textwidth}  
    \vspace{-1.em}
    \centering
    \includegraphics[width=.82\linewidth]{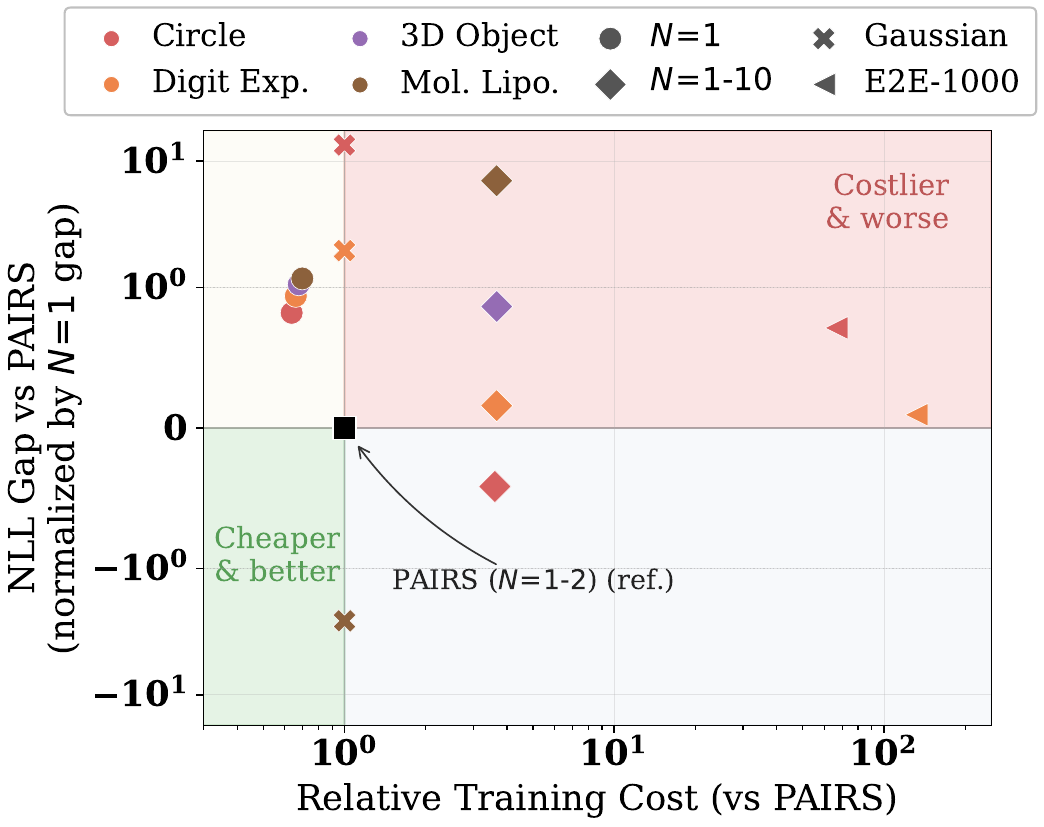}
    \caption{\small Normalized NLL gap vs. PAIRS against relative training cost. \textbf{End-to-end training at $\mathbf{N=1000}$ costs roughly $\mathbf{100\times}$ more than PAIRS for no consistent gain. $\mathbf{N=1}\textbf{-}\mathbf{10}$ is significantly costlier and worse on 3/4 tasks.}}
    \label{fig:pareto}
\end{wrapfigure}
Figure~\ref{fig:pareto} summarizes the cost--performance tradeoff by showing the NLL gap relative to PAIRS (normalized by the $N\!=\!1$ gap) against relative training cost. Across all tasks, PAIRS lies on or near the Pareto frontier: it matches or improves upon $N\!=\!1$ at comparable cost, while end-to-end training at fixed $N=1000$ requires $\approx 100\times$ more compute for no gain (measured memory and wall-clock costs per stage are reported in \autoref{app:compute-measured}). The $N\!=\!1\text{--}10$ strategy, despite costing $3\text{--}4\times$ more than PAIRS, lands in the ``costlier and worse'' quadrant on three of four tasks. The practical recommendation is clear: pretraining on $N \in \{1, 2\}$ with an expressive density-modeling objective offers the best trade-off between posterior quality, computational cost, and scalability to large $N$.
\vspace{0em}
\paragraph{Sensitivity to embedding dimension.}
\begin{wrapfigure}[14]{l}{0.4\textwidth}  
    \centering
    \includegraphics[width=.99\linewidth]{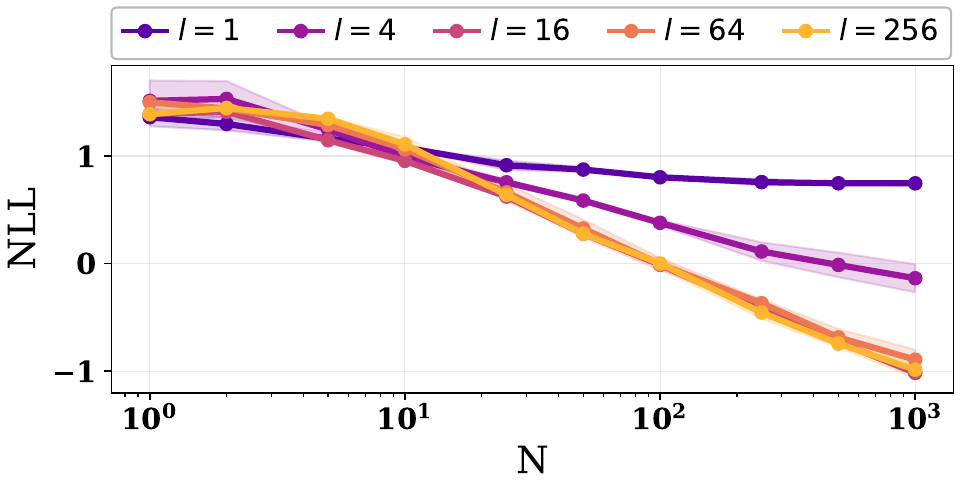}
    \caption{\small \textbf{Embedding dimension $l$ needs to be chosen sufficiently large.} Test NLL of PAIRS on Digit Expectation as a function of $N$, for varying $l$. Small embeddings ($l < 16$) underperform; performance plateaus for $l \geq 16$, consistent with Corollary~\ref{cor:ef}.}
    \label{fig:emb_size_effect}
\end{wrapfigure}

In \autoref{sec:bridging}, we argued that the embedding dimension $l$ is the key hyperparameter controlling PAIRS' representational capacity: per \autoref{thm:main}, $l$ must be at least the dimension $k$ of the canonical sufficient statistic of the EF approximation to the marginal likelihood. We study this empirically on Digit Expectation in \autoref{fig:emb_size_effect}, using a large pretraining set to isolate representational capacity from the backbone's inductive bias. Performance improves monotonically with $N$ at every tested $l$, but small embeddings ($l < 16$) substantially underperform, while performance plateaus for $l \geq 16$, confirming that $l$ need not grow to infinity, only large enough for the induced EF surrogate to faithfully approximate the marginal.  \autoref{app:emb_size_effect} extends this analysis to two architectures (Simple CNN and ResNet) and pretraining cardinalities (PAIRS, $N=1$): PAIRS is invariant to backbone choice and dominates $N=1$ across all configurations; $N=1$ approaches PAIRS only with the ResNet backbone at $l=256$, and even then at the cost of small-$N$ performance.
Competitive $N=1$ results therefore rely on backbone-specific effects that offer no principled handle on how to tune $l$ whereas PAIRS enforces the structure by construction.

\subsection{PAIRS for Conditional Generative modeling}\label{sec:image_gen}
Our final experiment evaluates PAIRS when both observations and targets lie in a high-dimensional space. We consider a procedural renderer that generates 3D scenes composed of 3 colored spheres at random positions under random directional illumination. Observations are $64{\times}64$ RGB images rendered from randomly sampled camera viewpoints, and the inference target is the image from a held-out camera angle. This task is inherently ambiguous when only a few views are available—many scenes are consistent with limited observations—but becomes nearly deterministic as the number of views increases. We replace the normalizing flow head with conditional flow matching~\citep{lipman2022flow}, which scales more naturally to high-dimensional image targets.
At $N=2$, on an H100-80GB GPU, we train with batches of 128 scenes. i.e., with 256 images per batch; at $N{=}100$ the same budget would permit only ${\sim}2$ scenes per batch, making convergence impractically slow. We therefore apply the two-stage PAIRS pipeline and amortize the flow matching head on datasets with $N \in [1, 100]$ at a cost that is independent from $N$. As shown in \autoref{fig:cond_generation}, reconstruction quality improves steadily with $N$: the mean square error decreases from $0.011$ at $N=1$ to $0.004$ at $N=100$ , yielding near-deterministic reconstructions—despite pretraining on only 1–2 views.  
\begin{figure}[h]
    \centering
    \sffamily
    \resizebox{\textwidth}{!}{
    \begin{tikzpicture}[
    every node/.style={inner sep=0, outer sep=0},
]

\def\tilesize{1.8}
\def\hgap{0.08}
\def\vgap{0.18}
\def\sectiongap{0.22}
\def\plotwidth{5}
\def\plotgap{0.50}

\pgfmathsetmacro{\totalh}{2 * \tilesize + \vgap}
\pgfmathsetmacro{\gridx}{\plotwidth + \plotgap}
\pgfmathsetmacro{\outx}{\gridx + \tilesize + \sectiongap}

\pgfmathsetmacro{\labely}{-\totalh - 0.25}
\pgfmathsetmacro{\plotcx}{\plotwidth / 2}
\pgfmathsetmacro{\gtx}{\outx + 5 * (\tilesize + \hgap)}
\pgfmathsetmacro{\gtcx}{\gtx + \tilesize/2}
\pgfmathsetmacro{\gridcx}{(\gridx + \gtx + \tilesize) / 2}

\node[font=\fontsize{9}{10}, anchor=north]
    at (\gridcx, \labely) {(b)};

\node[anchor=north west] (mseplot) at (0, .1) {%
    \includegraphics[width=\plotwidth cm, keepaspectratio]{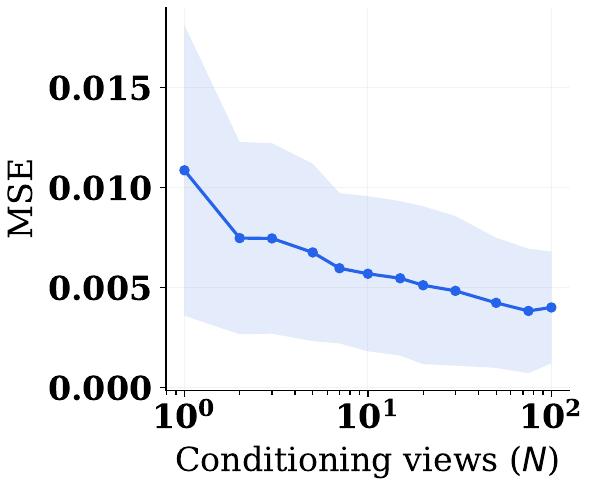}%
};
\node[font=\sffamily\fontsize{9}{10}\selectfont\color{labelgray}, anchor=north]
    at (\plotcx, \labely) {(a)};

\pgfmathsetmacro{\sepx}{\plotwidth + \plotgap*0.45}
\draw[lightgray, line width=0.25pt] (\sepx, 0.3) -- (\sepx, -\totalh - 0.25);

\pgfmathsetmacro{\condcx}{\gridx + \tilesize/2}
\node[font=\sffamily\fontsize{8}{9}\selectfont\color{inputamber!85!black}, anchor=south]
    at (\condcx, 0.2) {\textbf{Inputs}};

\foreach \i/\lbl in {0/{$N{=}1$}, 1/{$N{=}2$}, 2/{$N{=}5$}, 3/{$N{=}10$}, 4/{$N{=}100$}} {
    \pgfmathsetmacro{\cx}{\outx + \i * (\tilesize + \hgap) + \tilesize/2}
    \node[font=\sffamily\fontsize{8}{9}\selectfont\bfseries\color{labelgray}, anchor=south]
        at (\cx, 0.2) {\lbl};
}

\node[font=\sffamily\fontsize{8}{9}\selectfont\bfseries\color{gtgreen!80!black}, anchor=south]
    at (\gtcx, 0.2) {Target};

\pgfmathsetmacro{\arrs}{\outx + 0.12}
\pgfmathsetmacro{\arre}{\outx + 4*(\tilesize + \hgap) + \tilesize - 0.12}
\pgfmathsetmacro{\arrm}{(\arrs + \arre)/2}
\draw[-{Stealth[length=2.8pt,width=2.2pt]}, annogray, line width=0.4pt]
    (\arrs, 0.49) -- (\arre, 0.49);
\node[font=\sffamily\fontsize{9.5}{7}\selectfont\itshape\color{annogray}, anchor=south]
    at (\arrm, 0.5) {More views $\longrightarrow$ less uncertainty, better reconstruction};

\pgfmathsetmacro{\rowy}{-0 * (\tilesize + \vgap)}
\node[anchor=north west] at (\gridx, \rowy) {%
    \includegraphics[width=\tilesize cm, height=\tilesize cm]{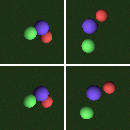}%
};
\draw[inputamber!60, line width=0.45pt, rounded corners=0.4pt]
    (\gridx, \rowy) rectangle ++(\tilesize, -\tilesize);
\pgfmathsetmacro{\tx}{\outx + 0 * (\tilesize + \hgap)}
\node[anchor=north west] at (\tx, \rowy) {%
    \includegraphics[width=\tilesize cm, height=\tilesize cm]{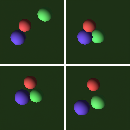}%
};
\pgfmathsetmacro{\tx}{\outx + 1 * (\tilesize + \hgap)}
\node[anchor=north west] at (\tx, \rowy) {%
    \includegraphics[width=\tilesize cm, height=\tilesize cm]{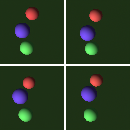}%
};
\pgfmathsetmacro{\tx}{\outx + 2 * (\tilesize + \hgap)}
\node[anchor=north west] at (\tx, \rowy) {%
    \includegraphics[width=\tilesize cm, height=\tilesize cm]{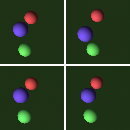}%
};
\pgfmathsetmacro{\tx}{\outx + 3 * (\tilesize + \hgap)}
\node[anchor=north west] at (\tx, \rowy) {%
    \includegraphics[width=\tilesize cm, height=\tilesize cm]{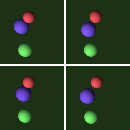}%
};
\pgfmathsetmacro{\tx}{\outx + 4 * (\tilesize + \hgap)}
\node[anchor=north west] at (\tx, \rowy) {%
    \includegraphics[width=\tilesize cm, height=\tilesize cm]{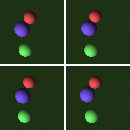}%
};
\node[anchor=north west] at (\gtx, \rowy) {%
    \includegraphics[width=\tilesize cm, height=\tilesize cm]{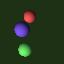}%
};
\draw[gtgreen, line width=0.5pt, rounded corners=0.4pt]
    (\gtx, \rowy) rectangle ++(\tilesize, -\tilesize);
\pgfmathsetmacro{\rowy}{-1 * (\tilesize + \vgap)}
\node[anchor=north west] at (\gridx, \rowy) {%
    \includegraphics[width=\tilesize cm, height=\tilesize cm]{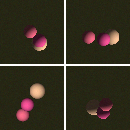}%
};
\draw[inputamber!60, line width=0.45pt, rounded corners=0.4pt]
    (\gridx, \rowy) rectangle ++(\tilesize, -\tilesize);
\pgfmathsetmacro{\tx}{\outx + 0 * (\tilesize + \hgap)}
\node[anchor=north west] at (\tx, \rowy) {%
    \includegraphics[width=\tilesize cm, height=\tilesize cm]{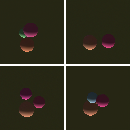}%
};
\pgfmathsetmacro{\tx}{\outx + 1 * (\tilesize + \hgap)}
\node[anchor=north west] at (\tx, \rowy) {%
    \includegraphics[width=\tilesize cm, height=\tilesize cm]{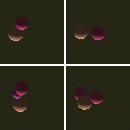}%
};
\pgfmathsetmacro{\tx}{\outx + 2 * (\tilesize + \hgap)}
\node[anchor=north west] at (\tx, \rowy) {%
    \includegraphics[width=\tilesize cm, height=\tilesize cm]{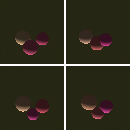}%
};
\pgfmathsetmacro{\tx}{\outx + 3 * (\tilesize + \hgap)}
\node[anchor=north west] at (\tx, \rowy) {%
    \includegraphics[width=\tilesize cm, height=\tilesize cm]{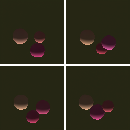}%
};
\pgfmathsetmacro{\tx}{\outx + 4 * (\tilesize + \hgap)}
\node[anchor=north west] at (\tx, \rowy) {%
    \includegraphics[width=\tilesize cm, height=\tilesize cm]{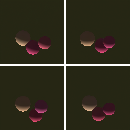}%
};
\node[anchor=north west] at (\gtx, \rowy) {%
    \includegraphics[width=\tilesize cm, height=\tilesize cm]{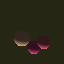}%
};
\draw[gtgreen, line width=0.5pt, rounded corners=0.4pt]
    (\gtx, \rowy) rectangle ++(\tilesize, -\tilesize);

\end{tikzpicture}
}
    \caption{\small \textbf{PAIRS scales to high-dimensional generative targets.} Novel-view synthesis of 3D scenes composed of colored spheres, from multiple rendered views. \textbf{(a)} MSE of the reconstructed target view as a function of the number of conditioning views $N$, decreasing from $0.011$ at $N{=}1$ to $0.004$ at $N{=}100$. \textbf{(b)} Samples from the posterior over the held-out view concentrate on the true target as $N$ grows, for two random scenes. End-to-end training at $N{=}100$ is computationally infeasible at the same batch size.}
    \label{fig:cond_generation}
    \vspace{-2em}
\end{figure}
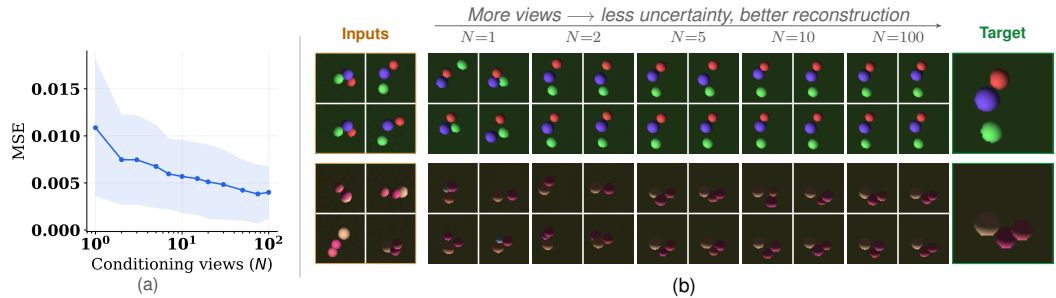

\section{Conclusion}
We have introduced PAIRS, a three-stage procedure that decouples the cost of training amortized posterior estimators from the deployment observation set cardinality. At its core is a simple recipe practitioners could plausibly stumble upon—train small, deploy large—whose correctness we explain as much as propose.
Our central result (\autoref{thm:main}) shows that under mild regularity, a mean-pool Deep Set encoder trained on sets of size at most two recovers the same representation as one trained at any larger cardinality, up to affine reparameterization. Corollary~\ref{cor:ef} characterizes this regime as implicitly assuming an exponential-family marginal likelihood---an idealization our experiments show is mild enough to extrapolate well on realistic problems.
Empirically, PAIRS matches the analytical posterior on a conjugate Gaussian model up to $N = 10^5$, tracks MCMC on a nuisance-coupled bump hunt, matches or surpasses end-to-end training on four realistic tasks at a fraction of the compute, and scales to high-dimensional image generation where end-to-end training is infeasible.

Our results open several avenues for future work. On the theoretical side, global sufficiency can only be approximated with finite data, and characterizing the interplay between intra-set variation (larger $N$ per training example) and inter-set variation (more distinct sets) might highlight when training with $N > 2$ is beneficial; we find that $N = 2$ is sufficient, not that it necessarily yields an optimal training trajectory for all tasks. On the practical side, we expect PAIRS to require substantially less training data than end-to-end approaches to reach good performance at large $N$, though we have not empirically verified this. Finally, extending PAIRS and the underlying theory to max-pooling, where a functional-equation argument may still apply, is a natural next step, while attention-based pooling is fundamentally harder as discussed in \autoref{sec:related}.

\section*{Acknowledgments}
The authors thank Philipp Windischhofer for providing the MCMC curves for the ``bump hunt'' task and thank both Philipp and Jay Sandesara for their early feedback. The authors would also like to thank Maria Cervera and Sorawit Saengkyongam for discussion and providing feedback on earlier version of the draft, as well as Andy Miller and Guillermo Sapiro for being supportive of the collaboration. MK is supported by the US Department of Energy (DOE) under Grant No.~DE-AC02-76SF00515. CP is supported by the UK Science and Technology Facilities Council (STFC) under Grant No.~ST/W000571/1. LH is supported by BMFTR Project SciFM 05D2025 and the Excellence Cluster ORIGINS, which is funded by the Deutsche Forschungsgemeinschaft (DFG, German Research Foundation) under Germany’s Excellence Strategy - EXC-2094-390783311.

\newpage

\bibliographystyle{plainnat}
\bibliography{biblio}
\newpage

\appendix


\section{Proof of Theorem~\ref{thm:main}}
\label{app:proof}
This appendix establishes three results that together formalize the claims made in \autoref{sec:method}.

\paragraph{Overview of results.}
\begin{description}
    \item[\autoref{thm:main} (pair training recovers a mean-pool sufficient statistic).] Under mild regularity on the embedder class and the assumption that the NPE problem is solved globally at $N \in \{1, 2\}$, the aggregate $\bigl(T_\omega(X_N), N\bigr)$ is sufficient for $\theta$ at every cardinality $N \geq 1$. This is the core guarantee underlying PAIRS.
    \item[Corollary~\ref{cor:ef} (EF marginals characterize mean-pool sufficiency).] Under standard regularity, the sufficiency assumption (M) of \autoref{thm:main} holds \emph{if and only if} the marginal likelihood factorizes as an exponential family with a per-element canonical statistic. In that case the learned embedder is, up to an affine map, the EF canonical statistic itself.
    \item[Corollary~\ref{cor:invariance} (training cardinalities $\geq 2$ are interchangeable).] Any pretraining schedule containing $N=1$ and at least one other cardinality $M \geq 2$ yields embedders in the same affine equivalence class. In particular, training at $\{1, 2\}$ gives the same representation as training at any larger $\{1, M\}$, so enlarging $N_{\mathrm{pre}}$ beyond $\{1, 2\}$ is not necessary.
\end{description}

\paragraph{High-level structure of the proofs.}
\autoref{thm:main} proceeds in five steps. (i) At $N=1$, global optimality of the NPE objective forces the sufficient statistic $t^\star$ to factor through the embedder: $t^\star = g \circ t_\omega$ for some continuous $g$. (ii) At $N=2$, sufficiency of the mean-pool aggregate yields an additive functional equation of the form $g(y_1) + g(y_2) = \tilde h(y_1 + y_2)$. (iii) Elimination of $\tilde h$ and centering around a reference point reduces this to the Cauchy additive equation $\bar g(u_1) + \bar g(u_2) = \bar g(u_1 + u_2)$, whose solutions are linear when $g$ is continuous, bounded over some interval, monotone over some interval, or measurable. (iv) The local affine identity $g(y) = Ay + b$ is then propagated to the full image of $t_\omega$ by a connectedness argument. (v) Affine composition with mean-pooling extends sufficiency to every $N$.

Corollary~\ref{cor:ef} follows from two classical results. The backward direction uses Fisher--Neyman factorization to show that any EF marginal with a per-element canonical statistic admits a fixed finite-dimensional sufficient statistic of mean-pool form at every $N$. The forward direction invokes the Koopman--Pitman--Darmois theorem~\citep{brown1986fundamentals}: any family admitting a fixed finite-dimensional sufficient statistic at every $N \geq 1$ must be exponential family, and exchangeability then forces the per-element structure.

Corollary~\ref{cor:invariance} reuses the machinery of \autoref{thm:main}. Step~(i) applies unchanged at $N=1$. At cardinality $M \geq 2$, fixing $M-2$ of the arguments to a reference point in $\mathrm{Im}(t_\omega)$ reduces the $M$-argument functional equation to the same two-argument Cauchy equation as in the main proof, after which steps (iii)--(v) apply verbatim.

\medskip
We now give the formal statement and proof.

\begin{theoremRef}[Pair training recovers a mean-pool sufficient statistic; formal version of \autoref{thm:main}]\label{thm:main-formal}
Let assumptions (M), (A1)--(A4), and (S) below hold. Then the aggregate $\bigl(T_\omega(X_N), N\bigr)$ is sufficient for $\theta$ at every cardinality $N \geq 1$, and the learned embedder satisfies $t^\star(x) = A t_\omega(x) + b$ for some $A \in \mathbb{R}^{k \times l}$ and $b \in \mathbb{R}^k$.
\end{theoremRef}
\paragraph{Assumptions.}
\begin{description}
\item[(M)] There exists a continuous map $t^\star:\mathcal X\to\mathbb R^k$
with $k\le l$ such that, for every $N\ge 1$, the pair
$\big(M_{t^\star}(X_N),N\big)=\big(\tfrac1N\sum_{i=1}^N t^\star(x_i),N\big)$
is minimal sufficient for $\theta$ given $X_N$ under the marginal
$p(X_N\mid\theta)=\int p(\psi)\prod_i p(x_i\mid\theta,\psi)\,\mathrm d\psi$.
\item[(A1)] The embedder class $\mathcal F$ contains every continuous
$\mathcal X\to\mathbb R^l$.
\item[(A2)] The density-estimator class $\mathcal Q$ is universal: for every
measurable conditional density $p(\theta\mid z,N)$ there exists $\phi$ with
$q_\phi(\theta\mid z,N)=p(\theta\mid z,N)$ $p$-a.e.
\item[(A3)] The optimum $t_\omega$ is continuous and a quotient map onto its
image in $\mathbb R^l$, with image of non-empty interior. (Satisfied, e.g.,
when $t_\omega$ terminates in an unconstrained full-rank linear layer.)
\item[(A4)] The NPE problem~\eqref{eq:NPE} is solved globally at
$N\in\{1,2\}$; equivalently,
$q_{\phi^\star}(\theta\mid M_{t_\omega}(X_N),N)=p(\theta\mid X_N)$ $p$-a.e.\
for $N\in\{1,2\}$.
\item[(S)] The prior $p(\theta, \psi)$ has full support on $\Theta \times \Psi$, and for every $(\theta, \psi)$ the conditional $p(x \mid \theta, \psi)$ has full support on $\mathcal{X}$. Consequently, for each $N \geq 1$ the marginal $p(X_N)$ has full support on $\mathcal{X}^N$.
\end{description}

\paragraph{Step 1: sufficiency at $N=1$ forces $t^\star=g\circ t_\omega$.\\}
At $N=1$, $M_{t_\omega}(\{x\})=t_\omega(x)$. By (A4) and (A2),
$q_{\phi^\star}(\theta\mid t_\omega(x),1)=p(\theta\mid x)$, so $t_\omega$ is
sufficient for $\theta$ given a single observation. By (M), $t^\star$ is
minimal sufficient at $N=1$; hence there exists a measurable
$g:\mathrm{Im}(t_\omega)\to\mathbb R^k$ with
\begin{equation}
t^\star(x)=g\big(t_\omega(x)\big)\qquad\forall x\in\mathcal X.
\label{eq:step1}
\end{equation}
Continuity of $g$ follows from continuity of $t^\star$ together with (A3):
$t_\omega$ being a continuous quotient map makes $g$ continuous via the
universal property of quotients.

\emph{Remark (a.e.\ to everywhere).}
Assumption~(A4) yields the identity $t^\star = g \circ t_\omega$ only
$p$-almost everywhere on $\mathcal{X}$. By~(S), the set on which it holds is
dense; since both sides are continuous (by~(M), (A1), (A3)), the identity
extends to all of $\mathcal{X}$. We apply the same upgrade tacitly in Step~2
for the Cauchy equation~\eqref{eq:cauchy-restricted}, and in Step~5 for the head
recovery.

\paragraph{Step 2: sufficiency at $N=2$ yields an additive functional
equation.\\} 
At $N=2$, (A4)+(A2) give
$q_{\phi^\star}(\theta\mid M_{t_\omega}(X_2),2)=p(\theta\mid X_2)$, so
$M_{t_\omega}(X_2)=\tfrac12(t_\omega(x_1)+t_\omega(x_2))$ is sufficient at
$N=2$. By (M), $M_{t^\star}(X_2)=\tfrac12(t^\star(x_1)+t^\star(x_2))$ is
minimal sufficient at $N=2$. Minimality gives a measurable
$\tilde h:\mathrm{Im}(t_\omega+t_\omega)\to\mathbb R^k$ with
$t^\star(x_1)+t^\star(x_2)=\tilde h\big(t_\omega(x_1)+t_\omega(x_2)\big)$.
Continuity of $\tilde h$ follows as in Step~1. Combining with
\eqref{eq:step1}:
\begin{equation}
g(y_1)+g(y_2)\;=\;\tilde h(y_1+y_2),
\qquad\forall y_1,y_2\in\mathrm{Im}(t_\omega).
\label{eq:cauchy-restricted}
\end{equation}

\paragraph{Step 3: from equation \eqref{eq:cauchy-restricted} to an affine form for $g$.\\}
We now solve equation~\eqref{eq:cauchy-restricted} for $g$ on $\mathrm{Im}(t_\omega)$.
By assumption~(A3), $\mathrm{Im}(t_\omega) \subset \mathbb{R}^l$ has non-empty
interior; fix a point $c$ in this interior and let $V \subset \mathbb{R}^l$ be
an open ball around $0$ small enough that $c + V \subset \mathrm{Im}(t_\omega)$
and $2c + V + V \subset \mathrm{Im}(t_\omega) + \mathrm{Im}(t_\omega)$.

\emph{(a) Eliminate $\tilde{h}$.\\}
Let $U' := \mathrm{Im}(t_\omega) - c \subset \mathbb{R}^l$; by (A3), $U'$ is open
and contains $0$. Choose $r > 0$ small enough that the open ball
$B(0, 2r) \subset U'$, and set $V := B(0, r)$. Then $V + V \subset B(0, 2r) \subset U'$.

For any $u_1, u_2 \in V$, the points $c + u_1$ and $c + u_2$ lie in
$\mathrm{Im}(t_\omega)$, so equation~\eqref{eq:cauchy-restricted} gives
\begin{equation}
g(c + u_1) + g(c + u_2) \;=\; \tilde{h}(2c + u_1 + u_2).
\label{eq:step3-star}
\end{equation}
Specializing~\eqref{eq:step3-star} to $u_2 = 0$ (valid since $0 \in V$) yields,
for all $v \in V$,
\begin{equation}
\tilde{h}(2c + v) \;=\; g(c + v) + g(c).
\label{eq:step3-htilde}
\end{equation}
Since $u_1 + u_2 \in V + V \subset U'$ and $U'$ is the domain on which
$\tilde h(2c + \cdot)$ is determined, we may apply~\eqref{eq:step3-htilde} at
$v = u_1 + u_2$, obtaining
\begin{equation}
g(c + u_1) + g(c + u_2) \;=\; g(c + u_1 + u_2) + g(c),
\qquad \forall\, u_1, u_2 \in V.
\label{eq:step3-elim}
\end{equation}

\emph{(b) Reduce to Cauchy's equation.\\} Define the centered map
\[
\bar{g}: V \to \mathbb{R}^k, \qquad
\bar{g}(u) \;:=\; g(c + u) - g(c).
\]
Subtracting $2 g(c)$ from both sides of~\eqref{eq:step3-elim} gives
\begin{equation}
\bar{g}(u_1) + \bar{g}(u_2) \;=\; \bar{g}(u_1 + u_2),
\qquad u_1, u_2 \in V,
\label{eq:step3-cauchy}
\end{equation}
which is Cauchy's additive equation on the open neighborhood $V$ of $0$.
Note that $\bar{g}(0) = 0$ by construction, and $\bar{g}$ is continuous on $V$
since $g$ is continuous (Step~1).

\emph{(c) Solve Cauchy under continuity.\\} A continuous solution of~\eqref{eq:step3-cauchy}
on a neighborhood of $0$ in $\mathbb{R}^l$ is necessarily the restriction of a
linear map \citep[see e.g.][Thm.~2.1.2]{aczel2006lectures}:
there exists $A \in \mathbb{R}^{k \times l}$ such that
\begin{equation}
\bar{g}(u) \;=\; A u, \qquad u \in V.
\label{eq:step3-linear}
\end{equation}
Equivalently, $g(c + u) = A u + g(c)$ for all $u \in V$, hence
\begin{equation}
g(y) \;=\; A y + b, \qquad b := g(c) - A c,
\qquad \text{for all } y \in c + V.
\label{eq:step3-affine-local}
\end{equation}

\emph{Remark (regularity).} Continuity is used here only to select the linear solution of Cauchy's equation; it is not essential to this step. Solutions of~\eqref{eq:step3-cauchy} are also linear if $\bar g$ is bounded on some open set, monotone (in the scalar case) on some interval, or measurable \citep[see e.g.][Ch.~2]{aczel2006lectures}, so any of these could replace continuity in~(c). Without any regularity, pathological (Hamel-basis) solutions exist. Continuity is, however, also used elsewhere in the proof (Steps~1, 2 and~4 and assumptions A1--A4), so this remark does not by itself extend the theorem to, e.g., discrete or piecewise-continuous statistics.

\emph{(d) Extend globally to $\mathrm{Im}(t_\omega)$.\\}
We propagate the local affine identity~\eqref{eq:step3-affine-local}
to all of $\mathrm{Im}(t_\omega)$ in two moves.

\emph{First, $\tilde{h}$ is affine on a neighborhood of $2c$.}
For any $y_1, y_2 \in c + V$, applying~\eqref{eq:cauchy-restricted} and
substituting the local form~\eqref{eq:step3-affine-local},
\begin{equation}
\tilde{h}(y_1 + y_2) \;=\; g(y_1) + g(y_2) \;=\; A(y_1 + y_2) + 2b.
\label{eq:step3-htilde-affine}
\end{equation}
Letting $z = y_1 + y_2$ range over $2c + (V + V)$, we obtain
$\tilde{h}(z) = A z + 2 b$ on the open set $W := 2c + (V+V)$.

\emph{Second, extend $g$.}
Fix any $y_1 \in \mathrm{Im}(t_\omega)$ such that there exists
$y_2 \in c + V$ with $y_1 + y_2 \in W$ — equivalently,
$y_1 \in W - (c + V) = c + (V + V)$. Applying~\eqref{eq:cauchy-restricted}
and~\eqref{eq:step3-htilde-affine},
\[
g(y_1) \;=\; \tilde{h}(y_1 + y_2) - g(y_2)
\;=\; A(y_1 + y_2) + 2b - (A y_2 + b)
\;=\; A y_1 + b.
\]
Hence $g = A\,(\cdot) + b$ on $\bigl(c + (V+V)\bigr) \cap \mathrm{Im}(t_\omega)$.

\emph{Iteration.} Repeating the two moves with $c + V$ replaced by
$c + (V + V)$ extends the affine identity to
$c + (V + V + V + V) \cap \mathrm{Im}(t_\omega)$, and so on. After $k$
iterations, $g$ is affine on $B(c, 2^k r) \cap \mathrm{Im}(t_\omega)$,
where $V = B(0, r)$. Since $\mathrm{Im}(t_\omega)$ is connected
(continuous image of the connected space $\mathcal{X}$), this yields
\begin{equation}
g(y) \;=\; A y + b, \qquad \forall\, y \in \mathrm{Im}(t_\omega).
\label{eq:step3-affine-global}
\end{equation}

\paragraph{Step 4: sufficiency at arbitrary $N$.\\} For any $N\ge 1$ and any
$X_N=\{x_i\}_{i=1}^N$,
\begin{equation}
\frac{1}{N}\sum_{i=1}^N t^\star(x_i)
\;=\;A\,\frac{1}{N}\sum_{i=1}^N t_\omega(x_i)+b
\;=\;A\,M_{t_\omega}(X_N)+b.
\label{eq:final}
\end{equation}
The map $(M_{t_\omega}(X_N),N)\mapsto(A\,M_{t_\omega}(X_N)+b,N)$ is an
invertible affine bijection (with the offset $b$ independent of $N$). Since
$(M_{t^\star}(X_N),N)$ is sufficient for $\theta$ at cardinality $N$ by (M),
so is $(M_{t_\omega}(X_N),N)$:
\[
p(\theta\mid X_N)=p\big(\theta\mid M_{t^\star}(X_N),N\big)
=p\big(\theta\mid M_{t_\omega}(X_N),N\big).
\]

\paragraph{Step 5: head recovery at deployment.\\} By (A2), refitting $\phi$
on the cached aggregates $\{(M_{t_\omega}(X_N^{(j)}),N^{(j)})\}_j$ at any
target cardinality distribution $p(N)$ recovers
$q_{\phi^\star_N}(\theta\mid z,N)=p(\theta\mid z,N)$ $p$-a.e.\ at that
cardinality. \qed

\paragraph{Remarks.}
\begin{enumerate}
\item The dimension bound $l\ge k$ is implicit in (A4): if $l<k$, no
$g:\mathbb R^l\to\mathbb R^k$ satisfying \eqref{eq:step1} can make $t_\omega$
sufficient at $N=1$, so (A4) fails.
\item The offset $Nb$ in $\sum_i t^\star(x_i)$ absorbs the
cardinality-dependent affine correction; this is exactly why the head must
be conditioned on $N$ in addition to $T_\omega(X_N)$.
\item Corollary~\ref{cor:ef} is immediate: under an EF sampling kernel with
$\psi$-independent canonical statistic $S(x)$, the marginal likelihood
factorizes through $\big(\sum_i S(x_i),N\big)$ (via the integral over
$\psi$), hence $t^\star=S$ satisfies (M).
\item The proof is stable: if (A4) holds only up to $\varepsilon$ in KL,
Hyers--Ulam stability of Cauchy's equation yields an affine $\tilde g$ with
$\|g-\tilde g\|_\infty \le C\varepsilon$ on compact subsets of
$\mathrm{Im}(t_\omega)$, and the conclusion of Theorem~\ref{thm:main} holds
up to a corresponding $O(\varepsilon)$ affine-equivalence error.
\end{enumerate}

\begin{corollary}[Exponential-family marginals characterize mean-pool sufficiency]
\label{cor:ef}
Assume the marginal $p(X_N\mid\theta) = \int p(\psi)\prod_{i=1}^N p(x_i\mid\theta,\psi)\,\mathrm{d}\psi$ is dominated, smooth in $\theta$, identifiable, and that the support of $x$ does not depend on $\theta$. Then assumption~(M) of Theorem~\ref{thm:main} holds if and only if the marginal admits the exponential-family factorization
\begin{equation}
\label{eq:ef-marginal}
p(X_N\mid\theta) \;=\; h_N(X_N)\,\exp\!\Big(\eta(\theta)^{\!\top}\!\sum_{i=1}^N S(x_i) \;-\; A(\theta,N)\Big),
\end{equation}
for some canonical statistic $S:\mathcal X\to\mathbb R^k$ (independent of $N$ and $\psi$), identifiable natural-parameter map $\eta:\Theta\to\mathbb R^k$, base measure $h_N$, and log-partition $A$. In that case $t^\star = S$, and any NPE-optimal embedder $t_\omega$ trained with $N_{\mathrm{pre}}=\{1,2\}$ satisfies
\[
S(x) \;=\; A\,t_\omega(x) + b, \qquad A\in\mathbb R^{k\times l},\; b\in\mathbb R^{k},
\]
so that $\big(T_\omega(X_N), N\big)$ is sufficient for $\theta$ at every cardinality $N\ge 1$.
\end{corollary}

\begin{proof}[Proof sketch]
($\Leftarrow$) Fisher--Neyman factorization applied to~\eqref{eq:ef-marginal} makes $\big(\tfrac{1}{N}\sum_i S(x_i),N\big)$ sufficient at every $N$; identifiability of $\eta$ promotes this to minimal sufficiency, so~(M) holds with $t^\star = S$ and Theorem~\ref{thm:main} applies verbatim. \\
($\Rightarrow$) Under the stated regularity, the Koopman--Pitman--Darmois theorem~\citep[see, e.g.,][Ch.~2]{brown1986fundamentals} states that any family admitting a fixed finite-dimensional sufficient statistic at every $N\ge 1$ is necessarily of exponential-family form. Exchangeability of the $x_i$ together with sufficiency of an additive statistic forces the canonical statistic to take the per-element form $\sum_i S(x_i)$, yielding~\eqref{eq:ef-marginal}.
\end{proof}

\begin{remark}
Condition~\eqref{eq:ef-marginal} is strictly broader than ``$p(x\mid\theta,\psi)$ is exponential family.'' It holds whenever $\psi$-integration preserves additive sufficiency in $\sum_i S(x_i)$---for instance when $p(x\mid\theta,\psi)$ is EF with canonical statistic $S(x)$ independent of $\psi$ and natural parameter decomposing as $\eta(\theta,\psi) = \eta_1(\theta) + \eta_2(\psi)$---and it covers curved EF submodels. The regularity hypothesis (support not depending on $\theta$) excludes order-statistic families such as $\mathrm{Uniform}(0,\theta)$, which admit fixed-dimensional sufficient statistics that cannot be written as a mean of per-element features; these correspond exactly to the regime $l\to\infty$ discussed in~\autoref{sec:bridging}\,(ii).
\end{remark}

\begin{corollary}[Training cardinalities $\geq 2$ are interchangeable]
\label{cor:invariance}
Adopt the assumptions of Theorem~\ref{thm:main}, replacing~(A4) by:
\begin{itemize}
\item[(A4$'$)] $\{1, M\} \subseteq \mathcal{N}_{\mathrm{pre}}$ for some
$M \geq 2$, and the NPE problem~\eqref{eq:NPE} is solved globally at
every $N \in \mathcal{N}_{\mathrm{pre}}$.
\end{itemize}
Then the learned embedder satisfies $t^\star(x) = A\, t_\omega(x) + b$ for
some $A \in \mathbb{R}^{k \times l}$, $b \in \mathbb{R}^k$, and
$\bigl(T_\omega(X_N), N\bigr)$ is sufficient for $\theta$ at every
$N \geq 1$. In particular, any two choices
$\mathcal{N}_{\mathrm{pre}}, \mathcal{N}_{\mathrm{pre}}'$ satisfying~(A4$'$)
yield embedders in the same affine equivalence class, and the finetuned
inference head converges to the same posterior estimator regardless of
which was used.
\end{corollary}

\begin{proof}[Proof sketch]
Step~1 of Theorem~\ref{thm:main} applies unchanged at $N = 1$, yielding a
continuous $g$ with $t^\star(x) = g(t_\omega(x))$ on $\mathcal{X}$.

For Step~2, we use cardinality $M$ in place of $N = 2$. Sufficiency of
$\bigl(M_{t_\omega}(X_M), M\bigr)$ and minimal sufficiency of $M_{t^\star}$
under~(M) yield a continuous $\tilde{h}_M$ with
\[
\textstyle\sum_{i=1}^M t^\star(x_i) \;=\; H_M\!\left(\sum_{i=1}^M t_\omega(x_i)\right),
\qquad H_M(z) := M\, \tilde{h}_M(z / M).
\]
Substituting $t^\star = g \circ t_\omega$ and writing $y_i := t_\omega(x_i)$,
\begin{equation}
\sum_{i=1}^{M} g(y_i) \;=\; H_M\!\left(\textstyle\sum_{i=1}^M y_i\right),
\qquad y_1, \ldots, y_M \in \mathrm{Im}(t_\omega).
\label{eq:cor2-mary}
\end{equation}
Fix any reference point $c \in \mathrm{Im}(t_\omega)$ and specialize
\eqref{eq:cor2-mary} to $y_3 = \cdots = y_M = c$:
\[
g(y_1) + g(y_2) + (M - 2)\, g(c) \;=\; H_M\!\left(y_1 + y_2 + (M - 2)\, c\right).
\]
Defining $\tilde{h}(z) := H_M\bigl(z + (M - 2)\, c\bigr) - (M - 2)\, g(c)$
recovers the two-argument Cauchy equation
\begin{equation}
g(y_1) + g(y_2) \;=\; \tilde{h}(y_1 + y_2),
\qquad y_1, y_2 \in \mathrm{Im}(t_\omega),
\label{eq:cor2-cauchy}
\end{equation}
identical to equation~\eqref{eq:cauchy-restricted} of Theorem~\ref{thm:main}.
Steps~3--5 of the original proof then apply verbatim, giving
$t^\star = A\, t_\omega + b$ and sufficiency of $(T_\omega(X_N), N)$ at
every $N \geq 1$.

The invariance claim follows: any two embedders $t_\omega, t_\omega'$
satisfying (A4$'$) each stand in affine correspondence with $t^\star$,
hence with each other, and the downstream head learns the same posterior
$p(\theta \mid T_\omega(X_N), N) = p(\theta \mid X_N)$ in both cases.
\end{proof}

\begin{remark}
Corollary~\ref{cor:invariance} formalizes the structural claim underlying
PAIRS: once the training set includes $N = 1$ and any $N = M \geq 2$,
enlarging $\mathcal{N}_{\mathrm{pre}}$ further cannot change the affine
equivalence class of the learned embedder. Training at varying $N$ up to
$N_{\max}$ --- the prevailing practice for DeepSet-based amortized
inference --- is therefore not necessary to obtain a representation that
generalizes to large $N$. Larger $M$ may still improve optimization
stability or gradient signal in practice, but does not change the
representational content of the optimum.
\end{remark}

\section{Additional results}

\subsection{Calibration of learned neural posterior estimators}\label{app:ACAUC}

Beyond NLL and relative MAE, a useful posterior surrogate should be \emph{calibrated}: credible regions derived from the posterior should match the empirical frequency of containing the true parameter. We report the absolute calibration AUC (ACAUC): for each credibility level $\alpha \in [0,1]$, we compute the empirical coverage of the estimated $\alpha$-credible region, and ACAUC is the absolute area between this coverage curve and the diagonal. A value of $0$ indicates a perfectly calibrated posterior, while $0.5$ corresponds to a constant estimator.

\begin{figure}[H]
    \centering
    \includegraphics[width=1.\linewidth]{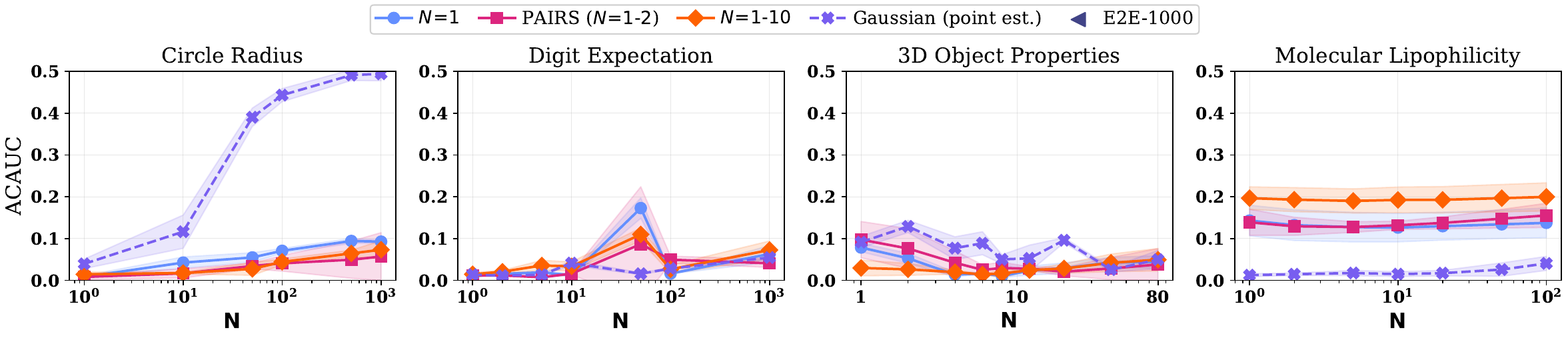}
    \caption{\textbf{Calibration across tasks and baselines.} Absolute calibration AUC (ACAUC) as a function of $N$; lower is better ($0 =$ perfectly calibrated, $0.5 =$ constant estimator).}
    \label{fig:ACAUC}
\end{figure}

\autoref{fig:ACAUC} shows that PAIRS remains well-calibrated as $N$ grows across all tasks, consistent with its strong NLL performance in the main text. The Gaussian baseline, by contrast, becomes progressively miscalibrated on Circle Radius, which is indicative of the model not being better than a constant predictor. 

\autoref{fig:BumpHuntCalib} similarly shows that the PAIRS neural posteriors are unbiased for several choices of the signal location $\psi$ and signal fraction $\theta$ with $N = 100$.

\begin{figure}[H]
    \centering
    \begin{subfigure}[b]{0.48\textwidth}
        \centering
        \includegraphics[width=\linewidth]{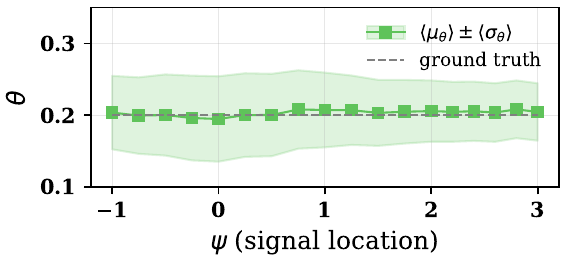}
        \caption{$\theta = 0.2$}
    \end{subfigure}
    \hfill
    \begin{subfigure}[b]{0.48\textwidth}
        \centering
    \includegraphics[width=\linewidth]{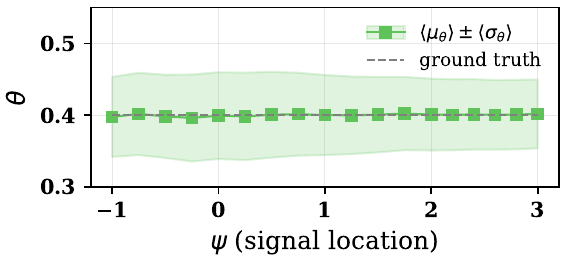}
        \caption{$\theta = 0.4$}
    \end{subfigure}
    \caption{\textbf{Calibration of the bump hunt task at $N = 100$.} The mean $\mu_\theta$ and standard deviation $\sigma_\theta$ of the neural posterior are shown, averaged over observations drawn from the prior at each signal location choice. The neural posterior is unbiased for the tested choices of signal location $\psi$ and signal fraction $\theta$.}
    \label{fig:BumpHuntCalib}
\end{figure}

\subsection{Effect of embedding size on the sufficiency of learned statistics for large N}\label{app:emb_size_effect}

We complement \autoref{sec:systematic_evaluation}, where we studied the effect of the embedding dimension $l$ on PAIRS with a single backbone for the Digit Expectation task, by extending the analysis to two architectures (Simple CNN and ResNet) and two pretraining cardinalities (PAIRS and $N=1$). This shows that the $l$-saturation behavior reported in the main text is not very sensitive to the choice of the specific backbone architecture. We also study the effect of $l$ on the Circle Radius task performance for both PAIRS and $N=1$. 

\begin{figure}[H]
    \centering
    \begin{subfigure}{0.48\linewidth}
        \centering
        \includegraphics[width=\linewidth]{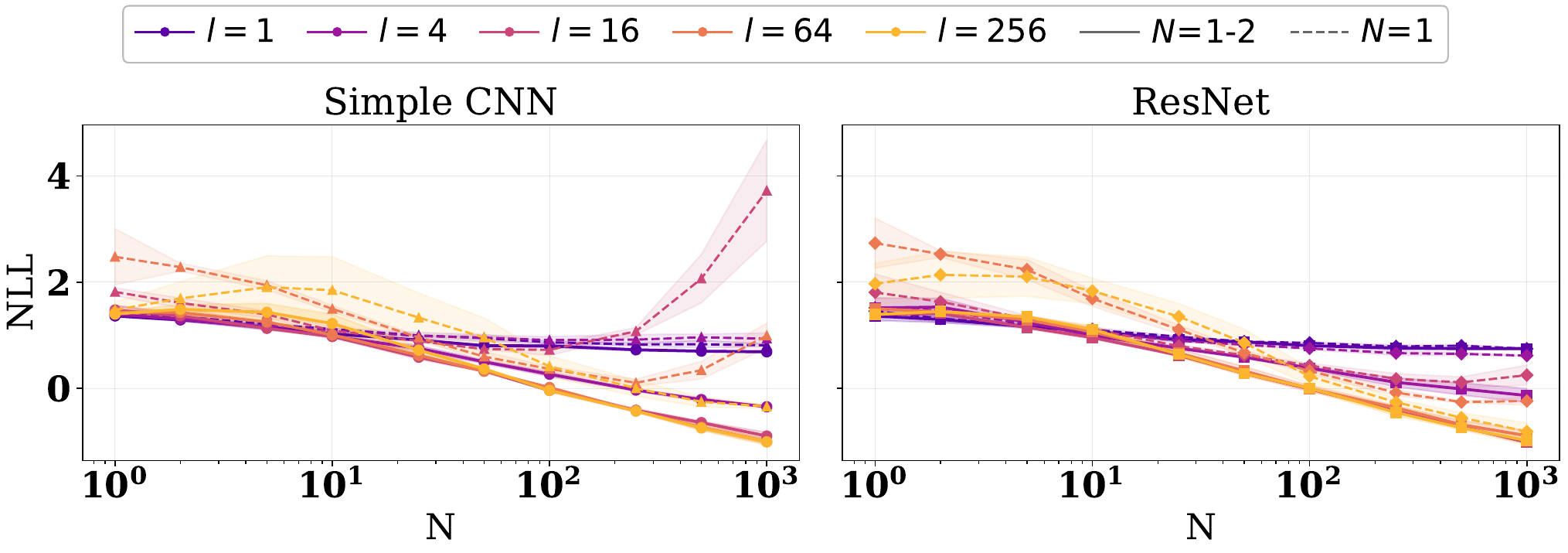}
        \caption{Digit Expectation -- Test NLL}
        \label{fig:emb_size_MNIST_NLL}
\end{subfigure}
\hfill
    \begin{subfigure}{0.48\linewidth}
        \centering
        \includegraphics[width=\linewidth]{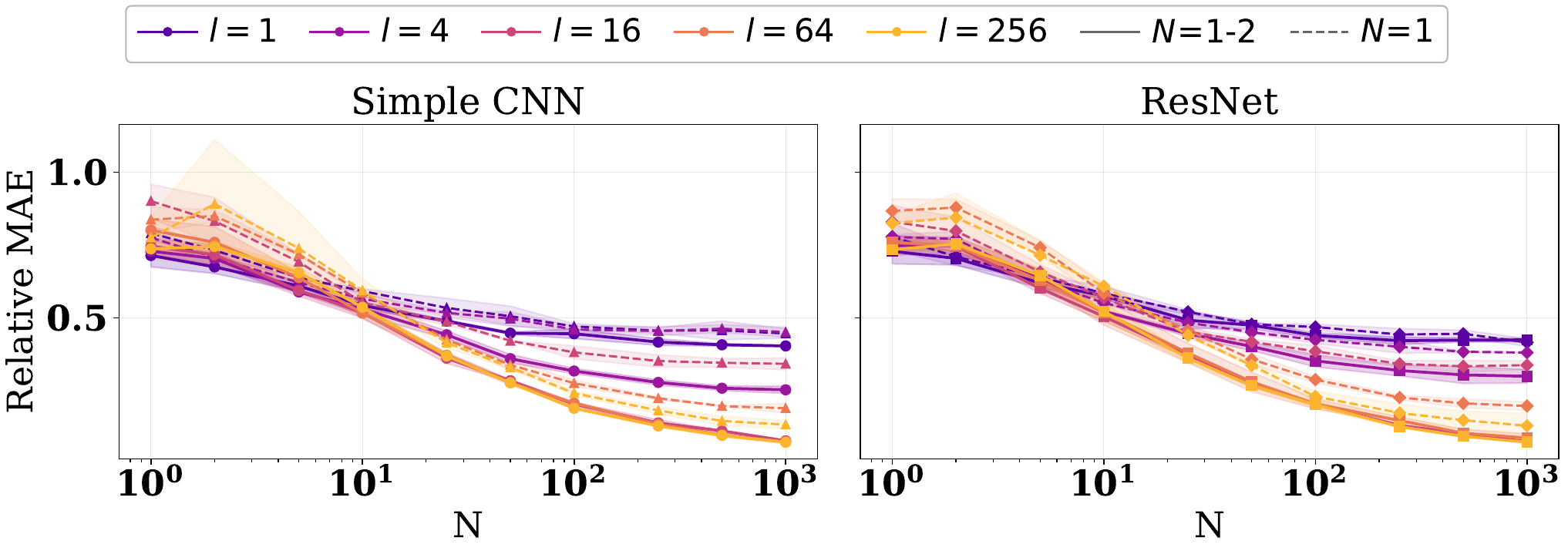}
        \caption{Digit Expectation -- Test relative MAE}
        \label{fig:emb_size_MNIST_MAE}
    \end{subfigure}
\\
\centering
\begin{subfigure}{0.48\linewidth}
        \centering
        \includegraphics[width=\linewidth]{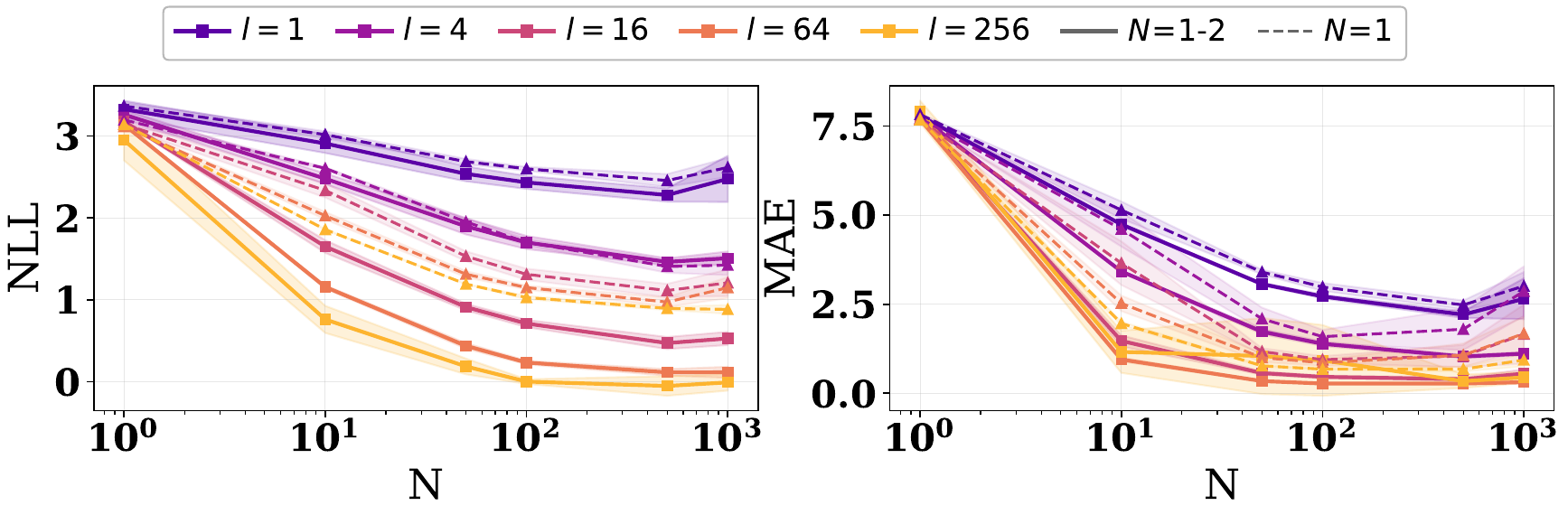}
        \caption{Circle Radius -- Test relative MAE and NLL}
        \label{fig:emb_size_circle_MAE}
    \end{subfigure}
    
    \caption{\textbf{Embedding-size sensitivity is driven by problem complexity, not backbone capacity.} Test performance on Digit Expectation for PAIRS (solid) and $N=1$ (dashed) across two backbones and $l \in \{1, 4, 16, 64, 256\}$. (a) NLL; (b) relative MAE.}
    \label{fig:emb_size_effect_app}
\end{figure}

Across both backbones, PAIRS saturates past $l \geq 16$ and dominates $N=1$ across all configurations. $N=1$ approaches PAIRS only with the ResNet at $l=256$, and even then at the cost of small-$N$ performance. The RMAE trends in (b) mirror the NLL trends in (a), confirming that the $l$-saturation effect is not an artefact of the likelihood-based metric. These findings support the claim in \autoref{sec:systematic_evaluation}: in our large-data regime, representational capacity is controlled by $l$ in a way that reflects problem complexity rather than model capacity.

\subsection{Numerical results underlying \autoref{fig:strategy-comparison}}\label{app:tables}

Tables~\ref{tab:results-circle-nll}--\ref{tab:results-lipo-rmae} report the test NLL and relative MAE (mean $\pm$ error) plotted in \autoref{fig:strategy-comparison}. E2E-$1000$ is only evaluated at its training cardinality.

\begin{table}[H]
\centering
\scriptsize
\caption{Circle Radius: Test NLL.}
\label{tab:results-circle-nll}
\resizebox{\linewidth}{!}{%
\begin{tabular}{lcccccc}
\toprule
Strategy & $N{=}1$ & $N{=}10$ & $N{=}50$ & $N{=}100$ & $N{=}500$ & $N{=}1000$ \\
\midrule
$N{=}1$ & $3.116\pm 0.015$ & $2.169\pm 0.030$ & $1.465\pm 0.041$ & $1.224\pm 0.021$ & $1.072\pm 0.046$ & $1.174\pm 0.086$ \\
\textbf{PAIRS ($N{=}1{\text{--}}2$)} & $3.137\pm 0.004$ & $1.340\pm 0.059$ & $0.557\pm 0.046$ & $0.392\pm 0.045$ & $0.206\pm 0.040$ & $0.346\pm 0.067$ \\
$N{=}1{\text{--}}10$ & $3.379\pm 0.004$ & $0.931\pm 0.029$ & $0.415\pm 0.057$ & $0.197\pm 0.049$ & $-0.046\pm 0.015$ & $0.050\pm 0.078$ \\
Gaussian & $3.639\pm 0.011$ & $3.745\pm 0.145$ & $5.877\pm 0.392$ & $6.782\pm 0.502$ & $9.422\pm 1.403$ & $11.469\pm 3.495$ \\
E2E-1000 & -- & -- & -- & -- & -- & $0.936\pm 0.131$ \\
\bottomrule
\end{tabular}}
\end{table}

\begin{table}[H]
\centering
\scriptsize
\caption{Circle Radius: Relative MAE.}
\label{tab:results-circle-rmae}
\resizebox{\linewidth}{!}{%
\begin{tabular}{lcccccc}
\toprule
Strategy & $N{=}1$ & $N{=}10$ & $N{=}50$ & $N{=}100$ & $N{=}500$ & $N{=}1000$ \\
\midrule
$N{=}1$ & $0.2193\pm 0.0012$ & $0.0833\pm 0.0057$ & $0.0328\pm 0.0011$ & $0.0246\pm 0.0010$ & $0.0289\pm 0.0063$ & $0.0471\pm 0.0099$ \\
\textbf{PAIRS ($N{=}1{\text{--}}2$)} & $0.2200\pm 0.0003$ & $0.0309\pm 0.0022$ & $0.0109\pm 0.0003$ & $0.0090\pm 0.0004$ & $0.0083\pm 0.0006$ & $0.0131\pm 0.0027$ \\
$N{=}1{\text{--}}10$ & $0.2231\pm 0.0016$ & $0.0175\pm 0.0006$ & $0.0090\pm 0.0005$ & $0.0072\pm 0.0005$ & $0.0055\pm 0.0002$ & $0.0061\pm 0.0008$ \\
Gaussian & $0.2199\pm 0.0007$ & $0.2044\pm 0.0134$ & $0.4267\pm 0.0282$ & $0.5361\pm 0.0318$ & $0.7963\pm 0.0513$ & $0.9054\pm 0.0626$ \\
E2E-1000 & -- & -- & -- & -- & -- & $0.0158\pm 0.0025$ \\
\bottomrule
\end{tabular}}
\end{table}

\begin{table}[H]
\centering
\scriptsize
\caption{Digit Expectation: Test NLL.}
\label{tab:results-digit-nll}
\resizebox{\linewidth}{!}{%
\begin{tabular}{lccccccc}
\toprule
Strategy & $N{=}1$ & $N{=}2$ & $N{=}5$ & $N{=}10$ & $N{=}50$ & $N{=}100$ & $N{=}1000$ \\
\midrule
$N{=}1$ & $1.229\pm 0.001$ & $1.225\pm 0.002$ & $1.203\pm 0.003$ & $0.969\pm 0.015$ & $1.495\pm 0.170$ & $0.460\pm 0.058$ & $0.209\pm 0.183$ \\
\textbf{PAIRS ($N{=}1{\text{--}}2$)} & $1.234\pm 0.002$ & $1.194\pm 0.006$ & $1.147\pm 0.001$ & $0.879\pm 0.009$ & $1.992\pm 2.378$ & $0.048\pm 0.009$ & $-1.009\pm 0.048$ \\
$N{=}1{\text{--}}10$ & $1.297\pm 0.008$ & $1.251\pm 0.008$ & $1.131\pm 0.011$ & $0.861\pm 0.004$ & $0.629\pm 0.054$ & $0.008\pm 0.007$ & $-0.888\pm 0.101$ \\
Gaussian & $1.328\pm 0.013$ & $1.291\pm 0.009$ & $1.190\pm 0.013$ & $1.320\pm 0.009$ & $1.221\pm 0.011$ & $1.281\pm 0.008$ & $1.378\pm 0.018$ \\
E2E-1000 & -- & -- & -- & -- & -- & -- & $-0.894\pm 0.038$ \\
\bottomrule
\end{tabular}}
\end{table}

\begin{table}[H]
\centering
\scriptsize
\caption{Digit Expectation: Relative MAE.}
\label{tab:results-digit-rmae}
\resizebox{\linewidth}{!}{%
\begin{tabular}{lccccccc}
\toprule
Strategy & $N{=}1$ & $N{=}2$ & $N{=}5$ & $N{=}10$ & $N{=}50$ & $N{=}100$ & $N{=}1000$ \\
\midrule
$N{=}1$ & $0.0747\pm 0.0001$ & $0.0734\pm 0.0002$ & $0.0715\pm 0.0002$ & $0.0593\pm 0.0005$ & $0.0540\pm 0.0034$ & $0.0410\pm 0.0008$ & $0.0379\pm 0.0026$ \\
\textbf{PAIRS ($N{=}1{\text{--}}2$)} & $0.0752\pm 0.0001$ & $0.0714\pm 0.0004$ & $0.0701\pm 0.0003$ & $0.0531\pm 0.0006$ & $0.0426\pm 0.0050$ & $0.0233\pm 0.0005$ & $0.0084\pm 0.0007$ \\
$N{=}1{\text{--}}10$ & $0.0798\pm 0.0015$ & $0.0746\pm 0.0006$ & $0.0686\pm 0.0008$ & $0.0521\pm 0.0004$ & $0.0378\pm 0.0025$ & $0.0224\pm 0.0001$ & $0.0085\pm 0.0005$ \\
Gaussian & $0.0809\pm 0.0011$ & $0.0770\pm 0.0011$ & $0.0708\pm 0.0010$ & $0.0738\pm 0.0011$ & $0.0738\pm 0.0009$ & $0.0767\pm 0.0008$ & $0.0844\pm 0.0012$ \\
E2E-1000 & -- & -- & -- & -- & -- & -- & $0.0098\pm 0.0007$ \\
\bottomrule
\end{tabular}}
\end{table}

\begin{table}[H]
\centering
\scriptsize
\caption{3D Object Properties: Test NLL.}
\label{tab:results-mv-nll}
\resizebox{\linewidth}{!}{%
\begin{tabular}{lccccccccc}
\toprule
Strategy & $N{=}1$ & $N{=}2$ & $N{=}4$ & $N{=}6$ & $N{=}8$ & $N{=}12$ & $N{=}20$ & $N{=}40$ & $N{=}80$ \\
\midrule
$N{=}1$ & $-2.664\pm 0.109$ & $-3.183\pm 0.096$ & $-3.676\pm 0.100$ & $-3.806\pm 0.107$ & $-3.885\pm 0.143$ & $-3.984\pm 0.109$ & $-4.069\pm 0.145$ & $-4.142\pm 0.141$ & $-4.167\pm 0.151$ \\
\textbf{PAIRS ($N{=}1{\text{--}}2$)} & $-2.767\pm 0.125$ & $-3.444\pm 0.083$ & $-3.962\pm 0.101$ & $-4.172\pm 0.140$ & $-4.236\pm 0.111$ & $-4.335\pm 0.150$ & $-4.435\pm 0.167$ & $-4.457\pm 0.207$ & $-4.446\pm 0.205$ \\
$N{=}1{\text{--}}10$ & $-2.588\pm 0.192$ & $-3.168\pm 0.185$ & $-3.689\pm 0.188$ & $-3.891\pm 0.149$ & $-3.962\pm 0.163$ & $-4.094\pm 0.109$ & $-4.219\pm 0.134$ & $-4.214\pm 0.108$ & $-4.205\pm 0.132$ \\
Gaussian & $-1.430\pm 0.224$ & $-1.862\pm 0.106$ & $-2.197\pm 0.073$ & $-2.420\pm 0.151$ & $-2.292\pm 0.127$ & $-2.310\pm 0.136$ & $-2.493\pm 0.541$ & $-2.177\pm 0.190$ & $-1.994\pm 0.691$ \\
\bottomrule
\end{tabular}}
\end{table}

\begin{table}[H]
\centering
\scriptsize
\caption{3D Object Properties: Relative MAE.}
\label{tab:results-mv-rmae}
\resizebox{\linewidth}{!}{%
\begin{tabular}{lccccccccc}
\toprule
Strategy & $N{=}1$ & $N{=}2$ & $N{=}4$ & $N{=}6$ & $N{=}8$ & $N{=}12$ & $N{=}20$ & $N{=}40$ & $N{=}80$ \\
\midrule
$N{=}1$ & $0.2143\pm 0.0084$ & $0.1657\pm 0.0046$ & $0.1404\pm 0.0057$ & $0.1319\pm 0.0094$ & $0.1242\pm 0.0093$ & $0.1163\pm 0.0073$ & $0.1116\pm 0.0077$ & $0.1080\pm 0.0073$ & $0.1065\pm 0.0083$ \\
\textbf{PAIRS ($N{=}1{\text{--}}2$)} & $0.2040\pm 0.0210$ & $0.1483\pm 0.0094$ & $0.1207\pm 0.0097$ & $0.1092\pm 0.0069$ & $0.1061\pm 0.0048$ & $0.1026\pm 0.0073$ & $0.0966\pm 0.0076$ & $0.0941\pm 0.0070$ & $0.0940\pm 0.0090$ \\
$N{=}1{\text{--}}10$ & $0.2273\pm 0.0206$ & $0.1813\pm 0.0235$ & $0.1418\pm 0.0158$ & $0.1247\pm 0.0175$ & $0.1167\pm 0.0159$ & $0.1089\pm 0.0114$ & $0.1009\pm 0.0109$ & $0.0989\pm 0.0088$ & $0.0955\pm 0.0081$ \\
Gaussian & $0.1990\pm 0.0204$ & $0.1639\pm 0.0074$ & $0.1431\pm 0.0112$ & $0.1266\pm 0.0107$ & $0.1284\pm 0.0090$ & $0.1269\pm 0.0066$ & $0.1136\pm 0.0108$ & $0.1349\pm 0.0047$ & $0.1410\pm 0.0186$ \\
\bottomrule
\end{tabular}}
\end{table}

\begin{table}[H]
\centering
\scriptsize
\caption{Molecular Lipophilicity: Test NLL.}
\label{tab:results-lipo-nll}
\resizebox{\linewidth}{!}{%
\begin{tabular}{lccccccc}
\toprule
Strategy & $N{=}1$ & $N{=}2$ & $N{=}5$ & $N{=}10$ & $N{=}20$ & $N{=}50$ & $N{=}100$ \\
\midrule
$N{=}1$ & $2.146\pm 0.502$ & $2.005\pm 0.429$ & $1.936\pm 0.399$ & $1.922\pm 0.404$ & $1.927\pm 0.417$ & $1.945\pm 0.434$ & $1.966\pm 0.455$ \\
\textbf{PAIRS ($N{=}1{\text{--}}2$)} & $1.952\pm 0.247$ & $1.805\pm 0.158$ & $1.733\pm 0.118$ & $1.730\pm 0.093$ & $1.758\pm 0.095$ & $1.820\pm 0.128$ & $1.887\pm 0.176$ \\
$N{=}1{\text{--}}10$ & $2.581\pm 0.473$ & $2.437\pm 0.491$ & $2.364\pm 0.512$ & $2.365\pm 0.550$ & $2.373\pm 0.562$ & $2.409\pm 0.582$ & $2.445\pm 0.601$ \\
Gaussian & $1.577\pm 0.040$ & $1.520\pm 0.029$ & $1.499\pm 0.037$ & $1.517\pm 0.048$ & $1.551\pm 0.060$ & $1.617\pm 0.069$ & $1.679\pm 0.071$ \\
\bottomrule
\end{tabular}}
\end{table}

\begin{table}[H]
\centering
\scriptsize
\caption{Molecular Lipophilicity: Relative MAE.}
\label{tab:results-lipo-rmae}
\resizebox{\linewidth}{!}{%
\begin{tabular}{lccccccc}
\toprule
Strategy & $N{=}1$ & $N{=}2$ & $N{=}5$ & $N{=}10$ & $N{=}20$ & $N{=}50$ & $N{=}100$ \\
\midrule
$N{=}1$ & $0.0531\pm 0.0032$ & $0.0505\pm 0.0042$ & $0.0488\pm 0.0047$ & $0.0480\pm 0.0048$ & $0.0479\pm 0.0049$ & $0.0478\pm 0.0049$ & $0.0478\pm 0.0048$ \\
\textbf{PAIRS ($N{=}1{\text{--}}2$)} & $0.0547\pm 0.0061$ & $0.0509\pm 0.0051$ & $0.0485\pm 0.0045$ & $0.0477\pm 0.0043$ & $0.0473\pm 0.0040$ & $0.0472\pm 0.0040$ & $0.0473\pm 0.0039$ \\
$N{=}1{\text{--}}10$ & $0.0537\pm 0.0037$ & $0.0512\pm 0.0036$ & $0.0494\pm 0.0036$ & $0.0490\pm 0.0037$ & $0.0487\pm 0.0036$ & $0.0486\pm 0.0035$ & $0.0486\pm 0.0034$ \\
Gaussian & $0.0553\pm 0.0021$ & $0.0527\pm 0.0015$ & $0.0519\pm 0.0019$ & $0.0529\pm 0.0023$ & $0.0546\pm 0.0028$ & $0.0577\pm 0.0031$ & $0.0603\pm 0.0034$ \\
\bottomrule
\end{tabular}}
\end{table}

\section{Experiment Details}\label{app:exp_details}

This appendix gathers the implementation details for every experiment in the paper. \autoref{app:shared-architecture} summarizes the architecture and training conventions used across all tasks. The remaining subsections mirror the experimental sections of the main text: the Gaussian validation of \autoref{sec:bridging} (\autoref{app:bivgauss}), the illustrative Bump Hunt of \autoref{sec:bump} (\autoref{app:bumphunt}), the four systematic-evaluation tasks of \autoref{sec:systematic_evaluation} (\autoref{app:systematic-eval}), and the novel-view synthesis experiment of \autoref{sec:image_gen} (\autoref{app:novelview}). \autoref{app:gaussian-baseline}--\ref{app:compute} document the Gaussian regression baseline, the evaluation metrics, and the compute budget.

\subsection{Shared Architectures}

All experiments follow the three-component architecture of \autoref{sec:method} unless otherwise noted:
\begin{enumerate}
    \item \textbf{Embedder} $t_\omega$ (task-specific): maps each observation $x_i$ to an $l$-dimensional representation $t_\omega(x_i) \in \mathbb{R}^l$.
    \item \textbf{Aggregator}: masked mean pooling over the $N$ per-observation embeddings, $T_\omega(X_N) = \tfrac{1}{N}\sum_{i=1}^N t_\omega(x_i)$, ignoring padded positions.
    \item \textbf{Inference head} $q_\phi$: a conditional density estimator conditioned on the aggregated context concatenated with an encoded observation count. We use a normalizing flow for all tasks except novel-view synthesis, where the target is an image, and we use conditional flow matching~\citep{lipman2022flow}.
\end{enumerate}

\subsection{Architecture and Training Details for Bivariate Gaussian and Bump Hunt Tasks}
\label{app:gauss-bump}

For the bivariate Gaussian task from~\autoref{sec:bridging} and the bump-hunt task from~\autoref{sec:bump}, the PAIRS embedder is a multi-layer perceptron (MLP) with three layers of 128 nodes and $l = 128$.
Pretraining $\nmax = 2$ is used in both cases.
The normalizing flow backbone consists of three stacked piecewise-quadratic
coupling transforms~\citep{DBLP:journals/corr/abs-1808-03856} interspersed with random permutations and ResNets~\citep{DBLP:journals/corr/HeZRS15}; the context MLP is composed of three layers and 128 outputs per layer. The AdamW optimizer is used, and a cosine annealing schedule is applied to the learning rate.
Training sample datasets are generated ``on-the-fly'' during training-, validation-, and test-time.

\subsubsection{Bivariate Gaussian}
\label{app:bivgauss}
This experiment validates the main theorem in a setting for which assumption~(M) holds \emph{exactly} and the analytical posterior is available in closed form.

\paragraph{Data generation.}
Observations are i.i.d.\ bivariate Gaussians $x_i \mid \theta,\psi \sim \mathcal{N}(\theta, \psi)$.
The Wishart prior over the covariance $\psi$ is determined by the scale matrix $W$ such that $W_{11} = 0.5, W_{12} = W_{21} = 0.25, W_{22} = 1$ and degrees of freedom $\nu = 5$.
A Normal-Wishart prior on the Gaussian locations is set with $\mu_0 = (-1, 2)$ and $\lambda_0 = 1$.
Via prior conjugacy, the posterior over the parameters of interest, $\theta_1$ and $\theta_2$, with $\psi$ marginalized out is analytically determined for any observation set.

\paragraph{Training.}
The PAIRS embedder is initially pretrained over 512 epochs of 256 batches, each comprising 1024 sample datasets of two observations, with an initial learning rate of $3 \times 10^{-4}$.
Each NPE head is finetuned at test-time $N$ with the same learning rate, number of batches, and batchsize used for the embedder pretraining.
Unlike the tasks in~\autoref{app:shared-architecture}, $N$ is not an input to the head; rather each head is trained for fixed $N$.
Five values of $N$ are considered: 2 (pretraining), 100, 1000, 10000, and 100000.

\subsubsection{Bump Hunt}
\label{app:bumphunt}

This experiment tests the efficacy of PAIRS in a scenario for which the posterior is non-EF but a direct comparison to the approximate ground truth is viable through MCMC, at least for the $N = 100$ case evaluated here.
It was first introduced by~\citet{Heinrich:2023bmt}.

\paragraph{Data generation.} Observations are drawn from a simple mixture model of two 1D Gaussians.
The ``signal'' Gaussian is narrow, but its location is controlled by the nuisance $\psi$: $x_i \mid \psi \sim \mathcal{N}(\psi, 0.1)$.
The ``background'' Gaussian is fixed but is broad relative to the signal: $x_i \sim \mathcal{N}(0, 1)$.
The parameter of interest is the signal fraction or mixture coefficient $\theta \in [0, 1]$, whose prior density is constant over the unit interval.

\paragraph{Training.} The PAIRS embedder pretraining consists of 200 epochs of 1024 batches with a batchsize of 2048; the initial learning rate is $10^{-3}$.
The NPE head is finetuned at $N = 100$ for 100 epochs composed of 8192 batches with a batchsize of 128.

\paragraph{MCMC.} The MCMC posterior estimate follows the procedure from~\citet{Heinrich:2023bmt}.

\subsection{Architecture and Training Details for All Other Tasks}
\label{app:shared-architecture}

\paragraph{$N$-encoding.}
The observation count is encoded as a scalar---either $N/N_{\max}$ (linear) or $\log(N)/\log(N_{\max})$ (logarithmic)---and concatenated to the aggregate $T_\omega(X_N)$ before conditioning the head. The encoding is \emph{not} learnable; the head learns to consume it during training.

\paragraph{Normalizing flow.}
When the target is low-dimensional, $q_\phi$ is a UMNN normalizing flow~\citep{wehenkel2019unconstrained}: a stack of monotonic autoregressive transformations with an autoregressive conditioner network and a monotonic integrand network.

\paragraph{Training loss.}
The training objective is the negative log-likelihood of the true parameters under the flow, averaged over the batch (Eq.~\eqref{eq:NPE}). All target parameters are standardized to zero mean and unit variance; the Jacobian correction $\sum_j \log \sigma_j$ is applied when reporting NLL in the original parameter space.

\paragraph{Two-stage pipeline.}
\emph{Pretraining}: all three components are trained jointly on small $N$ (strategy-dependent). During training, $N$ is sampled uniformly from $\{1,\dots,N_{\max}\}$ per batch element.
\emph{Amortized finetuning}: the embedder and aggregator are frozen; only the head is trained on pre-computed aggregated embeddings over the target-$N$ distribution. Because the frozen embedder is not traversed during finetuning, memory consumption per sample is negligible and the finetuning batch size is essentially a tuning choice rather than a hardware constraint.

\paragraph{Epoch scaling and gradient-step budget.}
For a fixed pretraining image budget $P$, the number of observation-sets available to a strategy with maximum cardinality $N_{\max}$ is $P / N_{\max}$. To keep the \emph{number of gradient steps} approximately equalized across strategies, we scale the number of pretraining epochs by $N_{\max}$. A consequence is that higher-$N_{\max}$ strategies see each of their (fewer) training sets $N_{\max}$ times more often than the $N=1$ strategy over the course of pretraining; this is a controlled design choice to isolate the effect of the pretraining cardinality rather than the total amount of data seen.

\paragraph{Optimizer.}
All experiments use AdamW; task-specific values of the learning rate, weight decay, and gradient clipping are reported below.

\subsection{Systematic evaluation (\autoref{sec:systematic_evaluation})}
\label{app:systematic-eval}

\autoref{sec:systematic_evaluation} compares PAIRS against baselines on four tasks covering diverse observation modalities (scalar, image, multi-view 3D, molecular graph). Each task exhibits a \emph{shared nuisance} $\psi$ that couples the observations within a set and must be marginalized to recover the target~$\theta$. This is the regime where PAIRS is expected to help: a single observation is insufficient to resolve~$\psi$, so pairs (or larger sets) are needed to expose the coupling and learn an aggregable sufficient statistic. The nature of $\psi$ differs across tasks:
\begin{itemize}
    \item \textbf{Circle Radius.} $\psi$ is the position of the target circle, which remains constant over observations within a set but is not of direct interest; only by seeing multiple images in which the target is fixed but the distractors vary can the model isolate~$\theta$.
    \item \textbf{Digit Expectation.} $\psi$ is the \emph{shared rotation angle} applied identically to all $N$ MNIST images in a set. The rotation creates a 6-vs-9 ambiguity that cannot be resolved from a single observation; multiple observations under the same rotation expose $\psi$ and the digit distribution.
    \item \textbf{3D Object Properties.} $\psi$ is implicit: it includes information concerning the shape of the object that is not directly of interest but influences the observed images. 
    \item \textbf{Molecular Lipophilicity.} $\psi$ is implicit: it encodes the molecular identity (e.g., 2D structure) shared across all conformers in a set, of which each $x_i$ is one thermal realization. LogP is a property of the molecule rather than of any individual conformer, so aggregating conformers averages out conformer-level fluctuations and isolates~$\theta$.
\end{itemize}
\paragraph{Common conventions.}
Each task is run with three seeds. Pretraining strategies $N=1$, $N{=}1{\text{-}}2$ (PAIRS) and $N{=}1{\text{-}}10$ share the same total image budget $P$; the per-strategy number of observation-sets is $P / N_{\max}$ and the number of pretraining epochs is scaled by $N_{\max}$ so that the total number of gradient steps is roughly equalized (see \autoref{app:shared-architecture}). Throughout, ``$M$ samples per~$N$'' at finetuning or test time refers to $M$ independent observation-sets of cardinality $N$, i.e.\ $M \cdot N$ total observations.

\subsubsection{Circle Radius}
\label{app:circle}

\paragraph{Data generation.}
The target is $r \sim \mathcal{U}(15, 50)$, the radius of a circle rendered on a $64{\times}64$ grayscale image, at a constant position ($\psi$). In addition to the target circle, each image contains $3$ distractor circles with positions and radii drawn independently from the same distribution. Within a set, the target circle is fixed while distractors are re-drawn independently across the $N$ images.

\paragraph{Dataset sizes.}
Pretraining uses a total budget of $100{,}000$ images, divided into $100{,}000 / N_{\max}$ observation-sets per strategy. Finetuning uses $5{,}000$ independent sets per finetuning cardinality~$N$ (i.e.\ $5{,}000 \cdot N$ images per cardinality), with $N \in \{10, 100, 1000\}$. Testing uses $500$ sets per~$N$. A $20\%$ validation split is used for early stopping.

\paragraph{Training configuration.}
Pretraining (Table~\ref{tab:circle-pretrain}) uses AdamW (weight decay $10^{-4}$), batch size $100$, learning rate $3{\times}10^{-4}$, and early stopping with patience $15$. Data augmentation consists of random $90^\circ$ rotations. Finetuning runs for $200$ epochs with batch size $50$, learning rate $10^{-4}$, and patience $15$.

\begin{table}[H]
\centering
\caption{Circle Radius: pretraining strategies. ``Epochs (scaled)'' is the base number of epochs ($50$) multiplied by $N_{\max}$ to equalize the total number of gradient steps across strategies.}
\label{tab:circle-pretrain}
\begin{tabular}{lccc}
\toprule
Strategy & $N_{\max}$ & Observation-sets & Epochs (scaled) \\
\midrule
$N = 1$        & 1  & $100{,}000$ & 50  \\
$N = 1$--$2$   & 2  & $50{,}000$  & 100 \\
$N = 1$--$10$  & 10 & $10{,}000$  & 500 \\
\bottomrule
\end{tabular}
\end{table}

\paragraph{Model architecture.}
The embedder is a $4$-layer CNN: \texttt{Conv}$(1 \to 64, k{=}5)$ $\!+\!$ \texttt{ReLU} $\!+\!$ \texttt{MaxPool} $\to$ \texttt{Conv}$(64 \to 128, k{=}5)$ $\!+\!$ \texttt{ReLU} $\!+\!$ \texttt{MaxPool} $\to$ \texttt{Conv}$(128 \to 256, k{=}3)$ $\!+\!$ \texttt{ReLU} $\!+\!$ \texttt{MaxPool} $\to$ \texttt{Conv}$(256 \to 256, k{=}3)$ $\!+\!$ \texttt{ReLU} $\!+\!$ \texttt{MaxPool} $\to$ \texttt{Linear}(flat $\to$ $l$), with embedding size $l = 10$. The flow is a single UMNN transformation with conditioner hidden $[128, 128]$, conditioner output size $32$, and integrand $[64, 64, 64]$. $N$-encoding is linear: $N/N_{\max}$.

\paragraph{Evaluation.}
We evaluate at $N \in \{1, 2, 5, 10, 25, 50, 100, 250, 500, 1000\}$ with $500$ test sets per~$N$ and $1{,}500$ posterior samples per test example.

\subsubsection{Digit Expectation (MNIST)}
\label{app:mnist}

\paragraph{Data generation.}
The target is the expected digit $\theta = \mathbb{E}[\mathrm{digit}] = \sum_{k=0}^{9} k\, p_k$, where $p \sim \mathrm{Dir}(\alpha)$ with $\alpha_6 = \alpha_9 = 2.5$ and $\alpha_k = 0.25$ for $k \notin \{6, 9\}$. Each sample consists of $N$ MNIST images drawn i.i.d.\ from~$p$. All $N$ images in a sample share the \emph{same} rotation angle $\psi \sim \mathcal{U}(0^\circ, 360^\circ)$ --- the nuisance --- and per-image Gaussian observation noise with $\sigma_i \sim \mathcal{U}(0.1, 0.3)$. Class $9$ images are rendered as $180^\circ$-rotated class $6$ images, so that single-image inference cannot distinguish between the two classes.

\paragraph{Dataset sizes.}
Pretraining budget: $250{,}000$ total images. Finetuning: $50{,}000$ independent sets per finetuning cardinality $N \in \{10, 100, 1000\}$. Test: $500$ sets per~$N$.

\paragraph{Training configuration.}
Pretraining uses AdamW (weight decay $10^{-4}$), batch size $256$, learning rate $10^{-4}$, cosine annealing with $5\%$ warmup, gradient clipping (max norm $1.0$), and patience $100$. The base number of epochs is $500$, scaled by $N_{\max}$ (Table~\ref{tab:mnist-pretrain}). Finetuning uses $400$ epochs, batch size $50$, learning rate $3{\times}10^{-4}$, patience $50$, gradient clipping $1.0$.

\begin{table}[H]
\centering
\caption{MNIST: pretraining strategies. Base epochs ($500$) are scaled by $N_{\max}$.}
\label{tab:mnist-pretrain}
\begin{tabular}{lccc}
\toprule
Strategy & $N_{\max}$ & Observation-sets & Epochs (scaled) \\
\midrule
$N = 1$        & 1  & $250{,}000$  & 500    \\
$N = 1$--$2$   & 2  & $125{,}000$  & 1{,}000 \\
$N = 1$--$10$  & 10 & $25{,}000$   & 5{,}000 \\
\bottomrule
\end{tabular}
\end{table}

\paragraph{Model architecture.}
The embedder is a compact ResNet: \texttt{Conv} stem $(1 \to 64, k{=}3)$ with \texttt{GroupNorm}, followed by $3$ stages of \texttt{BasicBlock}s ($64 \to 64 \to 128 \to 256$ with strides $1, 2, 2$ and one block per stage), adaptive average pooling, and a projection MLP $[256 \to 256 \to \texttt{ReLU} \to \texttt{Dropout}(0.1) \to l]$ with $l = 256$. \texttt{GroupNorm} (groups $= \min(32, C)$) replaces \texttt{BatchNorm} throughout to avoid evaluation-mode collapse under variable set sizes. The flow is a UMNN with $2$ transformations, conditioner $[256, 256]$, output size $64$, and integrand $[128, 128, 128]$. $N$-encoding is linear: $N/1000$.

\paragraph{Evaluation.}
We evaluate at $N \in \{1, 2, 5, 10, 25, 50, 100, 250, 500, 1000\}$ with $500$ test sets per~$N$ and $500$ posterior samples per test example.

\subsubsection{3D Object Properties (Multi-View)}
\label{app:mvcnn}

\paragraph{Data generation.}
The targets are two geometric properties (volume and log aspect ratio) of ModelNet40 objects~\citep{wu20153d}. The dataset contains approximately $12{,}000$ CAD meshes. Views are pre-rendered RGB images at $224 \times 224$ resolution; each image is the projection of a full object. Extra noise comes via \emph{partial-crop windows} used to occlude each view: at train, finetune, and test time we take a rectangular crop showing $35$--$50\%$ of each spatial dimension and fill the remainder with grey (the crop window is sampled independently per image). A sample of size $N$ thus consists of $N$ independently-cropped views of the \emph{same} object.

\paragraph{Augmentation regime.}
Because the embedder is frozen during finetuning, augmentation must be handled differently in the two stages:
\begin{itemize}
    \item \emph{Pretraining:} partial-crop augmentation is applied \emph{on-the-fly} every epoch---each epoch sees a fresh draw of crop windows for every image.
    \item \emph{Finetuning:} the embedder is frozen, so we pre-compute aggregated embeddings offline. To preserve the stochastic effect of augmentation, we perform $5$ augmentation passes per image and cache the resulting aggregates; each epoch of finetuning samples uniformly from these $5$ cached versions.
\end{itemize}
This gives the flow exposure to augmentation variance without the cost of re-running the frozen embedder.

\paragraph{Dataset sizes.}
Pretraining budget: $200{,}000$ rendered images in total, split across approximately $12{,}000$ ModelNet40 objects (so each object contributes roughly $17$ views). Finetuning: $100{,}000$ observation-sets per cardinality $N \in \{12, 40, 80\}$, each expanded by $5$ augmentation passes. Test: $5{,}000$ sets per~$N$.

\paragraph{Training configuration.}
Pretraining uses a base of $100$ epochs, learning rate $10^{-3}$, gradient clipping $1.0$, and patience $15$. Per-strategy batch sizes are $200$ (for $N{=}1$ and $N{=}1{\text{--}}2$), $100$ (for $N{=}1{\text{--}}5$), and $64$ (for $N{=}1{\text{--}}10$); these are dictated by GPU memory, because pretraining is the only stage that exercises the full ResNet-18 on $224{\times}224$ images. Finetuning uses $200$ epochs, batch size $64$, learning rate $10^{-4}$, and patience $15$.

\paragraph{Model architecture.}
The embedder is a ResNet-18 with \texttt{GroupNorm} (no \texttt{BatchNorm}): Conv $7 \times 7$ stem, $4$ stages ($64 \to 128 \to 256 \to 512$, $2$ BasicBlocks each), adaptive average pooling, and FC$(512 \to l)$ with $l = 16$. The flow is a UMNN with $2$ transformations, conditioner $[128, 128]$, output size $32$, and integrand $[64, 64, 64]$. $N$-encoding is logarithmic: $\log(N)/\log(N_{\max})$.

\paragraph{Evaluation.}
We evaluate at $N \in \{1, 2, 4, 6, 8, 12, 20, 40, 80\}$ with $5{,}000$ test sets per~$N$ and $500$ posterior samples per test example.

\subsubsection{Molecular Lipophilicity (Conformers)}
\label{app:conformers}

\paragraph{Data generation.}
The target is the Crippen logP of a drug-like molecule, computed from its SMILES via RDKit. Each sample consists of $N$ 3D conformers of the same molecule drawn from the GEOM-Drugs dataset~\citep{axelrod2022geom}, which contains $\approx 304$k molecules with conformers generated by CREST.

\paragraph{Dataset splits.}
Molecules are split $60\%$ pretrain / $20\%$ finetune / $20\%$ test with no molecule overlap between splits. At evaluation time we restrict to molecules whose number of available conformers is at least $\max_N$; this ensures that the test distribution over molecules does not shift with the evaluation cardinality (otherwise, larger $N$ would silently select for conformer-rich molecules).

\paragraph{Training configuration.}
Pretraining uses $200$ epochs, batch size $128$, learning rate $10^{-4}$, AdamW (weight decay $10^{-5}$), gradient clipping $1.0$, $5\%$ warmup, patience $30$, capped at $500$ batches per epoch. Finetuning uses $30$ epochs, batch size $128$, learning rate $3{\times}10^{-4}$, patience $15$, gradient clipping $1.0$, $5\%$ warmup.

\paragraph{Model architecture.}
The embedder is SchNet~\citep{schutt2017schnet} with $5$ interaction layers, $128$ hidden channels, $50$ Gaussians for distance expansion, and a cutoff of $10.0$\,\AA, followed by global mean pooling over atoms and \texttt{Linear}$(128 \to l)$ with $l = 128$. The flow is a single UMNN transformation with conditioner $[256, 256]$, output size $32$, and integrand $[128, 128, 128]$. $N$-encoding is logarithmic: $\log(N)/\log(N_{\max})$.

\paragraph{Evaluation.}
We evaluate at $N \in \{1, 2, 5, 10, 20, 50, 100, 200\}$ with $500$ posterior samples per test example.

\subsection{Novel-view synthesis (\autoref{sec:image_gen})}
\label{app:novelview}

The final experiment applies PAIRS to a conditional generative modeling problem in which both observations and targets are images. Here the inference head is not a normalizing flow but a U-Net trained with conditional flow matching.

\paragraph{Data generation.}
Each scene contains $3$ colored spheres placed in a unit cube, rendered by a GPU-vectorized procedural ray caster. Sphere centers are sampled uniformly in $[-0.8, 0.8]^3$ subject to a non-overlap constraint; radii $\sim \mathcal{U}(0.15, 0.4)$; and each sphere gets a random saturated color. Scene-level illumination consists of a directional light with azimuth $\sim \mathcal{U}(0, 2\pi)$, elevation $\sim \mathcal{U}(\pi/6, \pi/3)$, ambient coefficient $0.2$, and a per-channel background intensity $\sim \mathcal{U}(0.05, 0.2)$. Cameras use orthographic projection with azimuth $\sim \mathcal{U}(0, 2\pi)$ and elevation $\sim \mathcal{U}(\pi/8, 3\pi/8)$. Observations and targets are $64 \times 64$ RGB images. At training time, images are perturbed with Gaussian noise ($\sigma = 0.03$); at test time, renders are noise-free. The inference target is the image from a \emph{held-out} camera angle, conditioned on $N$ observed views and the target camera's extrinsics. Unlike the parameter-estimation tasks of~\autoref{sec:systematic_evaluation}, there is no clean $(\theta, \psi)$ split: the underlying scene (positions, radii, colors, illumination) is a shared latent that determines both the observed views and the target view. Posterior uncertainty over $\theta$ thus reflects how under-determined the scene is by $N$ views, and shrinks as $N$ grows.

\paragraph{Dataset sizes.}
The total pretraining image budget is $100{,}000$, divided into $100{,}000 / 2$ observation-sets. Each finetuning epoch renders $10{,}000$ fresh sets on the fly (the renderer is fast enough that on-line generation is preferable to caching). Testing uses $500$ sets per~$N$.

\paragraph{Model architecture.}
The embedder is a ResNet-style CNN with a $7 \times 7$ stem followed by strided \texttt{GroupNorm}/\texttt{SiLU} blocks, taking as input the RGB image concatenated with a per-pixel ray-direction map (input channels $= 6$); the output is an $l = 256$-dimensional embedding. A separate small MLP encodes the target-camera extrinsics into a $64$-dimensional vector. The aggregator mean-pools the observation embeddings and concatenates the aggregated embedding, the log-$N$ encoding, and the target-camera encoding into a conditioning vector.

The inference head is a U-Net velocity network trained with conditional flow matching~\citep{lipman2022flow}: input channels are $6$ (noised target image + per-pixel ray-direction map), time is injected as a sinusoidal embedding, and the conditioning vector is injected into every residual block via adaptive \texttt{GroupNorm}. The reference configuration uses base channels $192$, $3$ residual blocks per resolution, and time-embedding dimension $256$.

\paragraph{Training configuration.}
Pretraining runs at $N \in \{1, 2\}$ (PAIRS) with AdamW, learning rate $3 \times 10^{-4}$, cosine annealing, gradient clipping $1.0$, patience $30$, and $200$ epochs, batch size is $64$, dictated by U-Net memory at $64 \times 64$. Finetuning runs for $100$ epochs at batch size $64$ as well, learning rate $10^{-4}$, and patience $20$, on freshly rendered data each epoch (the embedder and aggregator are frozen, so each training example requires only one forward pass through the frozen networks plus a full U-Net step).

\paragraph{Hyperparameter search.}
Because this task is new and the end-to-end alternative is computationally infeasible at large $N$, we conducted a small grid search over the $N{=}1{\text{--}}2$ pretraining configuration at $64 \times 64$. The grid varied three factors:
\begin{itemize}
    \item U-Net base channels $\in \{32, 64, 96\}$;
    \item embedding dimension $l \in \{16, 32, 64\}$;
    \item pretraining learning rate $\in \{5 \times 10^{-4}, 10^{-3}, 3 \times 10^{-3}\}$.
\end{itemize}
The full factorial ($27$ configurations, single seed) was launched as a Bolt parent job and selected on validation MSE; the reference configuration reported above was obtained by extending the winning direction (higher capacity for both U-Net and embedder).

\paragraph{Sampling and evaluation.}
Generation uses a $50$-step Euler integration of the learned velocity field. Evaluation reports the per-pixel MSE of the sampled target view against the ground-truth held-out render, averaged over $500$ test scenes and evaluated at $N \in \{1, 2, 5, 10, 25, 50, 100\}$.

\subsection{Gaussian regression baseline}
\label{app:gaussian-baseline}

In \autoref{sec:systematic_evaluation} we compare PAIRS against a \emph{Gaussian regression} baseline that isolates the effect of the likelihood-based training objective. The baseline shares the embedder, aggregator, and $N$-encoding of PAIRS but replaces the flow head with an MLP trained under the mean-squared-error loss on $N \in \{1, 2\}$. After training, we estimate a per-dimension residual standard deviation $\hat\sigma = \mathrm{std}(\theta_{\mathrm{true}} - \theta_{\mathrm{pred}})$ on a held-out set. At test time the approximate posterior is $\mathcal{N}(\hat\theta, \hat\sigma^2 I)$---a fixed-width Gaussian centered on the point prediction. All metrics (NLL, relative MAE, ACAUC) are computed from $1{,}000$ samples drawn from this Gaussian.

\subsection{Evaluation metrics}
\label{app:metrics}

All metrics are computed from posterior samples of shape $(n_\mathrm{test}, n_\mathrm{samples}, n_\mathrm{params})$:

\begin{itemize}
    \item \textbf{Relative MAE.} Let $\hat\theta^{(j)}$ denote the posterior mean for test example $j$ and $\theta^{(j)}$ the true target. We report
    \[
        \mathrm{RMAE} \;=\; \frac{1}{R}\,\frac{1}{n_\mathrm{test}}\sum_{j=1}^{n_\mathrm{test}} \bigl|\,\hat\theta^{(j)} - \theta^{(j)}\,\bigr|,
    \]
    where $R$ is a task-dependent normalizing constant chosen as the range of the prior over $\theta$: $R = 35$ for Circle Radius (radii in $[15, 50]$), $R = 9$ for Digit Expectation (in $[0, 9]$), $R = 15$ for Molecular Lipophilicity (logP roughly in $[-5, 10]$), and $R = 1$ for 3D Object Properties (whose targets are standardized to unit variance, so the reported value coincides with the standardized MAE). For multi-parameter targets we average the per-parameter normalized errors.
    \item \textbf{NLL.} Negative log-likelihood of the true value under the learned head, including the Jacobian correction $\sum_j \log \sigma_j$ for the per-dimension parameter standardization used during training.
    \item \textbf{ACAUC.} Absolute calibration AUC, computed over $100$ uniformly spaced credible levels (see \autoref{app:ACAUC}). A value of $0$ indicates a perfectly calibrated posterior; $0.5$ corresponds to a constant estimator.
\end{itemize}

\subsection{Compute}
\label{app:compute}

All final-result experiments were run on single-GPU jobs (one H100-class GPU per job, $64$\,GB RAM, $256$\,GB disk). 

\paragraph{Pareto plot ``Relative Training Cost''.}
The horizontal axis of Figure~\ref{fig:pareto} reports the FLOP count of each strategy normalized by that of the PAIRS ($N{=}1{\text{--}}2$) run of the same task. Per-architecture FLOPs are counted analytically, layer-by-layer, on a forward pass; the forward count is then multiplied by $3$ to account for the gradient computation (forward plus two backward passes) and by the total number of gradient steps. For two-stage strategies (PAIRS, $N{=}1$, $N{=}1{\text{--}}10$), the total is the sum of (i) pretraining cost at average cardinality $(1 + N_{\max})/2$ and (ii) finetuning cost on the frozen aggregates at the three target cardinalities; the finetuning contribution does not depend on the pretraining strategy. For end-to-end strategies (E2E-$1000$), the FLOP cost scales linearly in the fixed number of observations per gradient step, yielding the approximately $100{\times}$ separation from PAIRS visible in the plot. 

\subsection{Measured per-stage training cost}
\label{app:compute-measured}

Tables~\ref{tab:compute-stagewise} and~\ref{tab:compute-summary} complement the FLOP counts of Figure~\ref{fig:pareto} with measured memory and time for each PAIRS stage (S1: pretraining on $N\le2$; S2: embedding and caching; S3: head finetuning) and for end-to-end (E2E) training.

\paragraph{Protocol.}
We time single AdamW steps (fp32) of the architectures of \autoref{app:exp_details} at the training batch size $B$ and set size $N$, on random inputs of the real shapes, averaging five steps after two warm-up steps. Peak memory includes weights, optimizer state and inputs. Measurements were made on one Apple M3 Pro GPU (28\,GiB), as our training cluster was no longer available; absolute times are not comparable to an H100, but ratios are. Steps that do not fit in memory are timed at the largest feasible batch and scaled linearly ($^\ddagger$), and their memory is estimated from the activations retained for the backward pass ($^\dagger$; within $20\%$ of the measured peak where both exist). GPU-hours are nominal: time per step times the full configured schedule, an upper bound given early stopping. E2E is reported only where it was trained in the paper (Circle Radius and Digit Expectation). Bivariate Gaussian and Bump Hunt, which use small MLP encoders on low-dimensional data, are omitted.

\paragraph{Discussion.}
PAIRS changes how training cost scales with $N$: S1 and S3 do not depend on $N_{\max}$, whereas each E2E step backpropagates through all $B N_{\max}$ observations. Memory is the binding constraint. On Digit Expectation, E2E-1000 needs 21.3\,GiB at $B{=}10$, while no PAIRS stage exceeds 1.3\,GiB at $B{=}256$; on an 80\,GB GPU, E2E-1000 fits at most 30--36 sets per batch, against more than $15{,}000$ for PAIRS pretraining. In nominal time, PAIRS is $23\times$ (Circle Radius) and $67\times$ (Digit Expectation) cheaper than E2E-1000. Caching (S2) takes under $4\%$ of PAIRS time on three tasks and $17\%$ on Molecular Lipophilicity; the cached embeddings live in host memory, and when a single target cardinality suffices only the per-set aggregates need storing ($51$\,MB instead of $51$\,GB for Digit Expectation at $N{=}1000$).

\begin{table}[t]
\centering\small
\caption{Measured per-stage cost of PAIRS and end-to-end (E2E) training (\autoref{app:compute-measured}). GPU-h: nominal schedule. $^\dagger$Estimated memory. $^\ddagger$Time scaled from the largest batch that fits. $^\S$Includes the frozen-encoder pass over all $B\times N$ views (no cache).}
\label{tab:compute-stagewise}
\setlength{\tabcolsep}{4pt}
\begin{tabular}{llrrrrrr}
\toprule
Task & Stage & $B$ & $N$ & Peak mem. & Time/step & Steps & GPU-h \\
\midrule
Circle Radius & S1 pretrain ($N{\le}2$) & 100 & 2 & 0.7\,GiB & 151\,ms & 200k & 8.4 \\
 & S2 embed \& cache & 1000 & 1 & 1.3\,GiB & 167\,ms & 6k & 0.3 \\
 & S3 finetune ($N{\le}1000$) & 50 & 1,000 & 0.9\,GiB & 12\,ms & 48k & 0.2 \\
 & E2E-1000 & 10 & 1,000 & 22.5\,GiB$^\dagger$ & 7.3\,s$^\ddagger$ & 100k & 202 \\
\midrule
Digit Expectation & S1 pretrain ($N{\le}2$) & 256 & 2 & 1.3\,GiB & 324\,ms & 390k & 35 \\
 & S2 embed \& cache & 64 & 1 & 0.2\,GiB & 7.5\,ms & 867k & 1.8 \\
 & S3 finetune ($N{\le}1000$) & 50 & 1,000 & 0.5\,GiB & 36\,ms & 960k & 9.6 \\
 & E2E-1000 & 10 & 1,000 & 21.3\,GiB & 5.6\,s & 2.00M & 3,119 \\
\midrule
3D Object Properties & S1 pretrain ($N{\le}2$) & 200 & 2 & 9.4\,GiB & 1.8\,s & 80k & 41 \\
 & S2 embed \& cache & 64 & 1 & 1.0\,GiB & 73\,ms & 62k & 1.2 \\
 & S3 finetune ($N{\le}80$) & 64 & 80 & 3.5\,GiB & 12\,ms & 750k & 2.6 \\
\midrule
Molecular Lipophilicity & S1 pretrain ($N{\le}2$) & 128 & 2 & 3.7\,GiB & 623\,ms & 100k & 17 \\
 & S2 embed \& cache & 256 & 1 & 0.6\,GiB & 277\,ms & 48k & 3.7 \\
 & S3 finetune ($N{\le}200$) & 128 & 200 & 99\,MiB & 11\,ms & 11k & 2\,min \\
\midrule
Novel-view synthesis & S1 pretrain ($N{\le}2$) & 64 & 2 & 22.5\,GiB & 6.8\,s & 125k & 237 \\
 & S3 finetune$^\S$ ($N{\le}100$) & 64 & 100 & 21.9\,GiB & 8.3\,s$^\ddagger$ & 16k & 36 \\
\bottomrule
\end{tabular}
\end{table}

\begin{table}[t]
\centering\small
\caption{PAIRS (S1+S2+S3) versus E2E at the deployment $N$. Peak memory: maximum over stages. Max $B$: largest batch fitting on an 80\,GB GPU (PAIRS: S1 at $N{=}2$). Symbols as in \autoref{tab:compute-stagewise}; ---: no E2E baseline.}
\label{tab:compute-summary}
\begin{tabular}{lrrrrrr}
\toprule
 & \multicolumn{3}{c}{PAIRS} & \multicolumn{3}{c}{E2E} \\
\cmidrule(lr){2-4}\cmidrule(lr){5-7}
Task & GPU-h & Peak mem. & Max $B$ (S1) & GPU-h & Peak mem. & Max $B$ \\
\midrule
Circle Radius & 8.8 & 1.3\,GiB & 15,140 & 202 & 22.5\,GiB$^\dagger$ & 30 \\
Digit Expectation & 47 & 1.3\,GiB & 18,428 & 3,119 & 21.3\,GiB & 36 \\
3D Object Properties & 45 & 9.4\,GiB & 1,718 & --- & --- & --- \\
Molecular Lipophilicity & 21 & 3.7\,GiB & 2,417 & --- & --- & --- \\
Novel-view synthesis & 274 & 22.5\,GiB & 252 & --- & --- & --- \\
\bottomrule
\end{tabular}
\end{table}


\newpage

\end{document}